\documentclass{article}

\usepackage[preprint, nonatbib]{neurips_2023}

\usepackage[utf8]{inputenc} %
\usepackage[T1]{fontenc}    %
\usepackage{hyperref}       %
\usepackage{url}            %
\usepackage{booktabs}       %
\usepackage{amsfonts}       %
\usepackage{nicefrac}       %
\usepackage{microtype}      %

\usepackage{caption}
\usepackage{subcaption}
\usepackage{wrapfig}
\usepackage[bottom]{footmisc}
\usepackage{header}

\newcommand{\citep}[1]{\cite{#1}}
\newcommand{\citet}[2][]{\cite[#1]{#2}}
\newcommand{\arxivonly}[1]{{#1}}
\newcommand{\arxivonlyverbatim}[1]{#1}
\newcommand{\confonly}[1]{}

\newcommand{\img}[1]{\includegraphics[width=0.17\textwidth]{#1}}

\title{Interpreting and Improving Diffusion Models from an Optimization Perspective}

\author{%
  Frank Permenter*\\
  Toyota Research Institute\\
  Cambridge, MA 02139 \\
  \texttt{frank.permenter@tri.global} \\
   \And{}
  Chenyang Yuan*\\
  Toyota Research Institute\\
  Cambridge, MA 02139 \\
  \texttt{chenyang.yuan@tri.global} \\
}

\begin{document}
\maketitle

\begin{abstract}
  Denoising is intuitively related to projection. Indeed, under the manifold
  hypothesis, adding random noise is approximately equivalent to orthogonal
  perturbation. Hence, learning to denoise is approximately learning to project.
  In this paper, we use this observation to interpret denoising diffusion models
  as approximate gradient descent applied to the Euclidean distance function. We
  then provide straight-forward convergence analysis of the DDIM sampler under
  simple assumptions on the projection error of the denoiser. Finally, we
  propose a new gradient-estimation sampler, generalizing DDIM using insights
  from our theoretical results. In as few as 5-10 function evaluations, our
  sampler achieves state-of-the-art FID scores on pretrained CIFAR-10 and CelebA
  models and can generate high quality samples on latent diffusion models.
\end{abstract}

\section{Introduction}
Diffusion models achieve state-of-the-art quality on many image
generation tasks
\citep{ramesh2022hierarchical,rombach2022high,saharia2022photorealistic}. They
are also successful in text-to-3D generation
\citep{poole2022dreamfusion} and novel view synthesis
\citep{liu2023zero}. Outside the image domain, they have been used for robot path-planning
\citep{chi2023diffusion}, prompt-guided human animation
\citep{tevet2022human},  and text-to-audio generation
\citep{kong2020diffwave}.

Diffusion models are presented as the reversal of a stochastic process that
corrupts clean data with increasing levels of random
noise~\citep{sohl2015deep,ho2020denoising}.  This reverse process can also be
interpreted as likelihood maximization of a noise-perturbed data distribution
using learned gradients (called \emph{score functions})
\citep{song2019generative,song2020score}. While these interpretations are
inherently probabilistic, samplers widely used in practice
(e.g. \cite{song2020denoising}) are often deterministic, suggesting diffusion
can be understood using a purely deterministic analysis. In this paper, we
provide such analysis by interpreting denoising as approximate
\emph{projection}, and diffusion as \emph{distance minimization} with gradient
descent, using the denoiser output as an estimate of the gradient. This in turn
leads to novel convergence results, algorithmic extensions, and paths towards
new generalizations.

\begin{figure*}
\centering
\begin{subfigure}[b]{0.33\textwidth}
  \centering
  \def\svgwidth{\textwidth}
  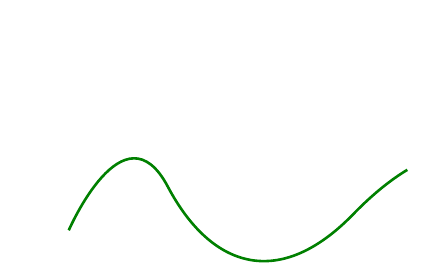
  \caption{Low noise}
  \label{fig:proj-low-noise}
\end{subfigure}%
\begin{subfigure}[b]{0.33\textwidth}
  \centering
  \def\svgwidth{\textwidth}
  \begingroup%
  \makeatletter%
  \providecommand\color[2][]{%
    \errmessage{(Inkscape) Color is used for the text in Inkscape, but the package 'color.sty' is not loaded}%
    \renewcommand\color[2][]{}%
  }%
  \providecommand\transparent[1]{%
    \errmessage{(Inkscape) Transparency is used (non-zero) for the text in Inkscape, but the package 'transparent.sty' is not loaded}%
    \renewcommand\transparent[1]{}%
  }%
  \providecommand\rotatebox[2]{#2}%
  \newcommand*\fsize{\dimexpr\f@size pt\relax}%
  \newcommand*\lineheight[1]{\fontsize{\fsize}{#1\fsize}\selectfont}%
  \ifx\svgwidth\undefined%
    \setlength{\unitlength}{112.36101646bp}%
    \ifx\svgscale\undefined%
      \relax%
    \else%
      \setlength{\unitlength}{\unitlength * \real{\svgscale}}%
    \fi%
  \else%
    \setlength{\unitlength}{\svgwidth}%
  \fi%
  \global\let\svgwidth\undefined%
  \global\let\svgscale\undefined%
  \makeatother%
  \begin{picture}(1,0.83485553)%
    \lineheight{1}%
    \setlength\tabcolsep{0pt}%
    \put(0,0){\includegraphics[width=\unitlength,page=1]{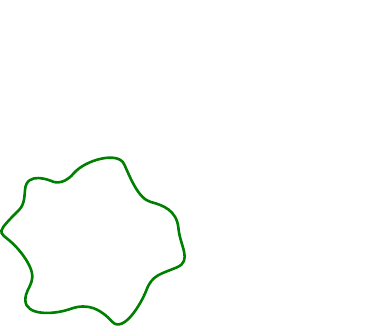}}%
    \put(0.22398777,0.17210834){\color[rgb]{0,0.50196078,0}\makebox(0,0)[lt]{\lineheight{1.25}\smash{\begin{tabular}[t]{l}$\mathcal{K}$\end{tabular}}}}%
    \put(0,0){\includegraphics[width=\unitlength,page=2]{figures/high_noise.pdf}}%
    \put(0.48161453,0.09108688){\makebox(0,0)[lt]{\lineheight{1.25}\smash{\begin{tabular}[t]{l}$x_0$\end{tabular}}}}%
    \put(0.82997489,0.80127343){\makebox(0,0)[lt]{\lineheight{1.25}\smash{\begin{tabular}[t]{l}$x_\sigma$\end{tabular}}}}%
    \put(0,0){\includegraphics[width=\unitlength,page=3]{figures/high_noise.pdf}}%
    \put(0.229652,0.60311875){\color[rgb]{1,0,0}\makebox(0,0)[lt]{\lineheight{1.25}\smash{\begin{tabular}[t]{l}$\proj_{\mathcal{K}}(x_\sigma)-x_\sigma$\end{tabular}}}}%
  \end{picture}%
\endgroup%

  \caption{High noise}
  \label{fig:proj-high-noise}
\end{subfigure}
\begin{subfigure}[b]{0.33\textwidth}
  \centering
  \def\svgwidth{\textwidth}
  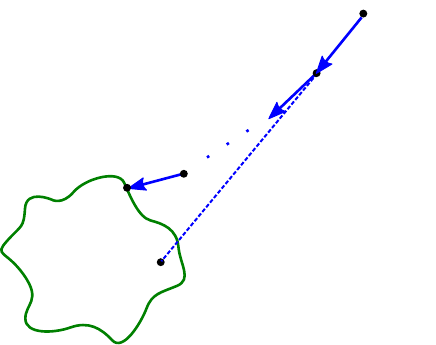
  \caption{Denoising process (DDIM)}
  \label{fig:proj-denoise}
\end{subfigure}
\caption{Denoising approximates projection: When $\sigma$ is small
  (\ref{fig:proj-low-noise}), most of the added noise lies in $\tan_\mathcal{K}(x_0)^\perp$
  with high probability under the manifold hypothesis. When $\sigma$ is large
  (\ref{fig:proj-high-noise}), both denoising and projection point in the same
  direction towards $\coneName$. We interpret the denoising process
  (\ref{fig:proj-denoise}) as minimizing $\distK^2(x)$ by iteratively taking
  gradient steps, estimating the direction of
  $\nabla \frac{1}{2}\distK^2(x) = x_t - \proj_\coneName(x_t)$ with
  $\epsilon_\theta(x_t)$.}
\label{fig:cartoons}
\end{figure*}

\paragraph{Denoising approximates projection} The core object in diffusion is a
\emph{learned denoiser} $\epsilon_\theta(x, \sigma)$, which, when given a noisy
point $x \in \mathbb{R}^n$ with noise level $\sigma > 0$, predicts the
\emph{noise direction} in $x$, i.e., it estimates $\epsilon$ satisfying
$x = x_0 + \sigma \epsilon$ for a clean datapoint $x_0$.

Prior work~\cite{rick2017one} interprets denoising as approximate
\emph{projection} onto the data manifold $\Kset \subseteq \mathbb{R}^n$.  Our
first contribution makes this interpretation rigorous by introducing a
relative-error model, which states that $x-\sigma \epsilon_{\theta}( x, \sigma)$
well-approximates the projection of $x$ onto $\Kset$ when $\sqrt{n} \sigma$
well-estimates the distance of $x$ to $\Kset$.  Specifically, we will assume
that
 \begin{align}\label{eq:relativeError}
  \| x-\sigma \epsilon_{\theta}( x, \sigma) - \projK(x)\| \le \eta \distK(x)
 \end{align}
when $(x, \sigma)$ satisfies $\frac{1}{\nu}\distK(x) \le \sqrt{n}\sigma \le \nu \distK(x)$
for constants $1 > \eta  \ge 0$ and $\nu \ge 1$.

This error model is motivated by the following theoretical observations
that hold when $\sigma \approx \distK(x)/\sqrt{n} $:
\begin{enumerate}
\item When $\sigma$ is small and the manifold hypothesis holds, denoising
  approximates projection given that most of the added noise is orthogonal to the data
  manifold; see \Cref{fig:proj-low-noise} and
  \Cref{prop:OracleDenoisingUnderReachInformal}.
\item When $\sigma$ is large, then any denoiser predicting any weighted mean of
  the data $\Kset$ has small relative error; see \Cref{fig:proj-high-noise} and
  \Cref{prop:rel-error-large-noise}.
\item Denoising with the \emph{ideal denoiser} is a $\sigma$-smoothing of
  $\projK(x)$ with relative error that can be controlled under mild assumptions;
  see \Cref{sec:ideal-denoiser-error}.
\end{enumerate}

We also empirically validate this error model on sequences
$(x_i, \sigma_i)$ generated by pretrained diffusion samplers, showing that it
holds in practice on image datasets.

\paragraph{Diffusion as distance minimization}
Our second contribution analyzes diffusion sampling under the error
model~\eqref{eq:relativeError}. In particular, we show it is equivalent to
approximate gradient descent on the \emph{squared} Euclidean distance function
$f(x) := \frac{1}{2} \distK^2(x)$, which satisfies $\nabla f(x) = x-\projK(x)$.  Indeed, in
this notation, \eqref{eq:relativeError} is equivalent to
\[
  \| \nabla f(x)-\sigma \epsilon_{\theta}( x, \sigma)\| \le \eta \| \nabla f(x)\|,
\]
a standard relative-error assumption used in gradient-descent analysis.  We also
show how the error parameters $(\eta, \nu)$ controls the schedule of noise
levels $\sigma_t$ used in diffusion sampling. \Cref{thm:DDIMWithError} shows
that with bounded error parameters, a geometric $\sigma_t$ schedule guarantees
decrease of $\distK(x_t)$ in the sampling process. Finally, we leverage
properties of the distance function to design a sampler that aggregates previous
denoiser outputs to reduce gradient-estimation error (\Cref{sec:newAlgorithms}).

We conclude with computational evaluation of our sampler
(\Cref{sec:experiments}) that demonstrates state-of-the-art FID scores on
pretrained CIFAR-10 and CelebA datasets and comparable results to the best
samplers for high-resolution latent models such as Stable Diffusion
\citep{rombach2022high} (\Cref{fig:sd-figures}). \Cref{sec:discussion} provides
novel interpretations of existing techniques under the framework of distance
functions and outlines directions for future research.

\section{Background}\label{sec:background}
Denoising diffusion models (along with all other generative models) treat
datasets as samples from a probability distribution $D$ supported on a subset
$\coneName$ of $\mathbb{R}^n$.  They are used to \emph{generate} new points in
$\coneName$ outside the training set. We overview their basic features. We then
state properties of the Euclidean distance function $\distK(x)$ that are key to
our contributions.

\subsection{Denoising Diffusion Models}
\paragraph{Denoisers}
Denoising diffusion models are trained to estimate a \emph{noise vector}
$\epsilon \in \mathbb{R}^n$ from a given noise level $\sigma > 0$ and noisy
input $x_\sigma \in \mathbb{R}^n$ such that $x_\sigma = x_0 + \sigma \epsilon$
approximately holds for some $x_0$ in the data manifold $\coneName$.  The
learned function, denoted
$\epsilon_{\theta} : \mathbb{R}^n \times \mathbb{R}_+ \rightarrow \mathbb{R}^n$,
is called a \emph{denoiser}. The trainable parameters, denoted jointly by
$\theta \in \mathbb{R}^m$, are found by (approximately) minimizing the following
loss function using stochastic gradient descent:
\begin{align}\label{eq:trainingLoss}
  L(\theta): = \Ex{}
  \norm{\epsilon_\theta(x_0 +  \sigma_t \epsilon, \sigma_t) - \epsilon}^2
\end{align}
where the expectation is taken over $x_0 \sim D$, $t \sim [N]$, and
$\epsilon \sim \mathcal{N}(0, I)$.  Given noisy $x_\sigma$ and noise level
$\sigma$, the denoiser $\epsilon_{\theta}(x_\sigma, \sigma)$ induces an estimate
of $\hat x_0 \approx x_0$ via
\begin{align}\label{eq:x0def}
  \hat x_0(x_\sigma, \sigma) := x_\sigma - \sigma \epsilon_\theta(x_\sigma, \sigma).
\end{align}

\paragraph{Ideal Denoiser} The \emph{ideal denoiser} $\eps^*(x_\sigma, \sigma)$
for a particular noise level $\sigma$ and data distribution $D$ is the minimizer
of the loss function
\begin{align*}
  \Loss(\eps^*) = \Ex_{x_0 \sim D} \Ex_{x_\sigma \sim \mathcal{N}(x_0, \sigma I)}
  \norm{(x_\sigma - x_0)/\sigma - \eps^*(x_\sigma, \sigma)}^2.
\end{align*}
Informally, $\hat x_0$ predicted by $\epsilon^*$ is an estimate of the expected
value of $x_0$ given $x_\sigma$. If $D$ is supported on a set $\Kset$, then
$\hat x_0$ lies in the \emph{convex hull} of $\Kset$.

\paragraph{Sampling}\label{sec:sampling}
Aiming to improve accuracy, sampling algorithms construct a sequence
$\hat x^t_0 := \hat x_0(x_t, \sigma_t)$ of estimates from a sequence of points
$x_t$ initialized at a given $x_N$.  Diffusion samplers iteratively construct
$x_{t-1}$ from $x_t$ and $\epsilon_\theta(x_t, \sigma_t)$, with a monotonically
decreasing \emph{$\sigma$ schedule} $\braces{\sigma_t}_{t=N}^0$. For simplicity of
notation we use $\epsilon_\theta(\cdot, \sigma_t)$ and
$\epsilon_\theta(\cdot, t)$ interchangeably based on context.

For instance, the randomized DDPM \citep{ho2020denoising} sampler uses the
update
\begin{align}\label{eq:DDPM-scaled}
  x_{t-1} &= x_t + (\sigma_{t'} - \sigma_t) \epsilon_\theta(x_t, \sigma_t) + \eta w_t,
\end{align}
where $w_t \sim \mathcal{N}(0, I)$, $\sigma_{t'} = \sigma_{t-1}^2/\sigma_t$ and
$\eta = \sqrt{\sigma_{t-1}^2 - \sigma_{t'}^2}$ (we have
$\sigma_{t'} < \sigma_{t-1} < \sigma_{t}$, as $\sigma_{t-1}$ is the geometric
mean of $\sigma_{t'}$ and $\sigma_{t}$). The deterministic DDIM
\citep{song2020denoising} sampler, on the other hand, uses the update
\begin{align}\label{eq:DDIM-scaled}
  x_{t-1} &= x_t + (\sigma_{t-1} - \sigma_t) \epsilon_\theta(x_t, \sigma_t).
\end{align}
See \Cref{fig:proj-denoise} for an illustration of this denoising process. Note
that these samplers were originally presented in variables $z_t$ satisfying
$z_t = \sqrt{\alpha_t} x_t$, where $\alpha_t$ satisfies
$\sigma^2_t = \frac{1-\alpha_t}{\alpha_t}$.  We prove equivalence of the
original definitions to \eqref{eq:DDPM-scaled} and \eqref{eq:DDIM-scaled} in
\Cref{sec:scaled} and note that the change-of-variables from $z_t$ to $x_t$
previously appears in
\cite{song2020score,karras2022elucidating,song2020denoising}.

\subsection{Distance and Projection}
The \emph{distance function}
to a set $\coneName \subseteq \mathbb{R}^n$ is defined as
\begin{align}\label{eq:dist-def}
  \distK(x) := \inf \{ \norm{x-x_0} : x_0 \in \coneName \}.
\end{align}
The \emph{projection} of $x \in \mathbb{R}^n$, denoted $\projK(x)$,  is the set
of points that attain this distance:
\begin{align}\label{eq:proj-def}
  \projK(x) := \{ x_0 \in \coneName : \distK(x) = \norm{x-x_0} \}.
\end{align}
When $\projK(x)$ is unique, i.e., when $\projK(x) = \{x_0\}$, we abuse notation
and let $\projK(x)$ denote $x_0$. Then $x - \projK(x)$ is the direction
of steepest descent of $\distK(x)$:
\begin{prop}[page 283, Theorem 3.3 of~\cite{delfour2011shapes}]\label{prop:distanceProperties}
  Suppose $\coneName \subseteq \mathbb{R}^n$ is closed and $x \notin
  \coneName$. Then $\projK(x)$ is unique for almost all $x \in \mathbb{R}^n$
  (under the Lebesgue measure). If $\projK(x)$ is unique, then
  $\nabla \distK(x)$ exists, $\|\nabla \distK(x)\| = 1$ and
  \begin{align*}
    \nabla \ihalf \distK(x)^2 &= \distK(x), \\
    \nabla \distK(x) &= x-\projK(x).
  \end{align*}
\end{prop}
In addition, we define a \emph{smoothed squared-distance function} for a
smoothing parameter $\sigma > 0$ by using the $\softmin_{\sigma^2}$ operator instead
of $\min$.
\begin{align*}
  \distop^2_\Kset(x, \sigma)
  &:= \substack{\softmin_{\sigma^2} \\ x_0 \in \Kset} \norm{x_0 - x}^2 \\
  &= {\textstyle -\sigma^2  \log\paren{\sum_{x_0 \in \Kset}
    \exp\paren{-\frac{\norm{x_0 - x}^2}{2\sigma^2}}}}.
\end{align*}
In contrast to $\distK^2(x)$, $\distK^2(x, \sigma)$ is always
differentiable and lower bounds $\distK^2(x)$.

\section{Denoising as Approximate Projection}\label{sec:denoiseVsProj}

\begin{figure}
  \centering
  \confonly{
    \includegraphics[width=0.45\textwidth]{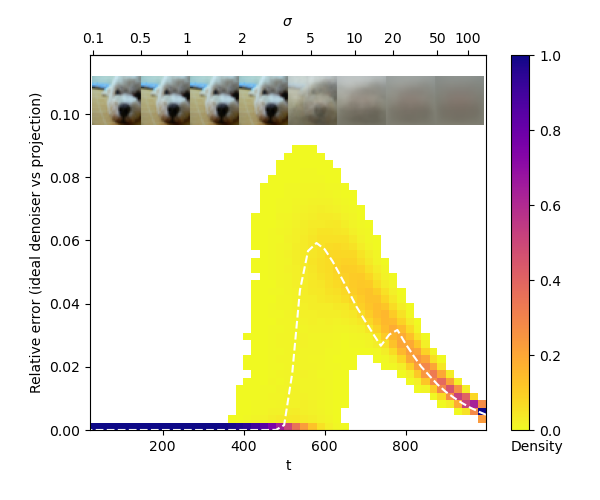}
  }
  \arxivonly{
    \includegraphics[width=0.60\textwidth]{figures/ideal_denoiser_error}
  }
  \caption{Ideal denoiser well-approximates projection onto the CIFAR-10
    dataset. Dashed line plots error for the example shown, and density plot
    shows the error distribution over 10k different DDIM sampling
    trajectories. }
  \label{fig:ideal-denoiser}
\end{figure}

In this section, we provide theoretical and empirical justifications for our
relative error model, formally stated below:
\begin{defn}\label{def:main}
  We say $\epsilon_{\theta}(x, \sigma)$ is an \emph{$(\eta, \nu)$-approximate} projection
  if there exists constants $1 > \eta \ge 0$ and $\nu \ge 1$ so that
  for all $x$ with unique $\projK(x)$ and for all $\sigma$ satisfying
  $\frac{1}{\nu}\distK(x) \le \sqrt{n}\sigma \le \nu \distK(x)$, we have
 \[
  \| x-\sigma \epsilon_{\theta}( x, \sigma) - \projK(x)\| \le \eta \distK(x).
 \]
\end{defn}
To justify this model, we will prove relative error bounds under different
assumptions on $\epsilon_{\theta}$, $(x, \sigma)$ and $\coneName$.  Analysis of
DDIM based on this model is given in \Cref{sec:sampling}. Formal statements and
proofs are deferred to \Cref{sec:formalDenoising}.
\Cref{sec:futher-experiments} contains further experiments verifying our error
model on image datasets.

\subsection{Relative Error Under the Manifold Hypothesis}
The \emph{manifold hypothesis}
\citep{bengio2013representation,fefferman2016testing,pope2021intrinsic} asserts
that ``real-world'' datasets are (approximately) contained in low-dimensional
manifolds of $\mathbb{R}^n$. Specifically, we suppose that $\coneName$ is a
manifold of dimension $d$ with $d \ll n$.  We next show that denoising is
approximately equivalent to projection, when noise is small compared to the
\emph{reach} of $\coneName$, defined as the largest $\tau > 0$ such that
$\projK(x)$ is unique when $\distK(x) < \tau$.

The following classical result tells us that for small perturbations $w$,
 the difference between $\projK(x_0 + w)$ and $x_0$ is contained
in the \emph{tangent space} $\tan_\coneName(x_0)$, a subspace of dimension $d$
associated with each $x_0 \in \coneName$.
\begin{lem}[Theorem 4.8(12) in \cite{federer1959curvature}]\label{lem:tubularProj}
  Consider $x_0 \in \coneName$ and $w \in \tan_\coneName(x_0)^{\perp}$. If
  $\|w\| < \reach(\coneName)$, then $\projK(x_0+w) = x_0$.
\end{lem}
When $n \gg d$, the orthogonal complement $\tan_\coneName(x)^{\perp}$ is large
and will contain most of the weight of a random $\epsilon$, intuitively
suggesting $\projK(x_0+\sigma \epsilon) \approx x_0$ if $\sigma$ is small. In
other words, we intuitively expect that the \emph{oracle denoiser}, which
returns $x_0$ given $x_0 + \sigma \epsilon$, approximates projection.

We formalize this intuition using Gaussian concentration
inequalities~\cite{vershynin2018high} and Lipschitz continuity of $\projK(x)$
when $\distK(x) < \reach(\coneName)$.

\begin{prop}[Oracle denoising
  (informal)]~\label{prop:OracleDenoisingUnderReachInformal} Given
  $x_0 \in \coneName$, $\sigma > 0$ and $\epsilon \sim \mathcal{N}(0, I)$, let
  $x_\sigma = x_0 + \sigma \epsilon \in \R^n$. If $n \gg d$ and
  $\sigma \sqrt{n} \lesssim \reach(\coneName)$, then with high probability,
  \begin{align*}
    \textstyle
    \norm{\projK(x_\sigma) - x_0} \lesssim \sqrt{\frac{d}{n}} \distK(x_\sigma)
  \end{align*}
  and for $\nu \lesssim 1 + \sqrt{\frac{d}{n}}$,
  $\frac{1}{\nu} \distK(x_\sigma) \le \sqrt{n} \sigma \le \nu \distK(x_\sigma)$.
\end{prop}
Observe this motivates both constants $\eta$ and $\nu$ used by our relative error model.
 We note
prior analysis of diffusion under the manifold hypothesis is given
by~\cite{de2022convergence}.

\subsection{Relative Error in Large Noise Regime}
We next analyze denoisers in the large-noise regime, when $\distK(x)$ is much
larger than $\diam(\Kset) := \sup \{ \|x-y\| : x, y \in \coneName\}$. In this
regime (see \Cref{fig:proj-high-noise} for an illustration), any denoiser that
predicts a point in the convex hull of $\Kset$
approximates projection with  error small compared to $\distK(x)$.
\begin{prop}~\label{prop:rel-error-large-noise} Suppose
  $x-\sigma \epsilon_{\theta}(x, \sigma) \in \conv(\coneName)$. If $\sqrt{n} \sigma \le \nu \distK(x)$, then
  $\norm{x - \sigma\epsilon_{\theta}(x, \sigma) - \projK(x)} \le\nu  \frac{\diam(\Kset)}{\sqrt{n} \sigma} \distK(x)$.
\end{prop}
Thus any denoiser that predicts any weighted mean of the data, for instance the
ideal denoiser, well approximates projection when
$\sqrt{n} \sigma \gg \diam(\Kset)$. For most diffusion models used in practice,
$\sqrt{n} \sigma_N$ is usually 50-100 times the diameter of the training set,
with $\sqrt{n} \sigma_t$ in this regime for a significant proportion of
timesteps.

\subsection{Relative Error of Ideal Denoisers} \label{sec:ideal-denoiser-error}

We now consider the setting where $\coneName$ is a finite set and
$\eps_{\theta}(x, \sigma)$ is the ideal denoiser $\eps^*(x, \sigma)$. We first
show that predicting $\hat{x}_0$ with the ideal denoiser is equivalent to
projection using the $\sigma$-smoothed distance function:
\begin{prop} \label{prop:smoothed-ideal-denoiser}
  For all $\sigma > 0$ and $x \in \R^n$, we have
  \begin{align*}
    \nabla_x \ihalf \distop^2_\Kset(x, \sigma) = \sigma \eps^*(x, \sigma).
  \end{align*}
\end{prop}
This shows our relative error model is in fact a bound on
\[
\|\nabla_x  \distop^2_\Kset(x, \sigma) -  \nabla_x \distop^2_\Kset(x)\|,
\]
the error between the gradient of the smoothed distance function and that of the
true distance function.  Since the amount of smoothing is directly determined by
$\sigma$, it is therefore natural to bound this error using $\distK(x)$ when
$\sqrt{n} \sigma \approx \distK(x)$. In other words,
\Cref{prop:smoothed-ideal-denoiser} directly motivates our error-model.

Towards a rigorous bound, let
$N_\alpha(x)$ denote the subset of $x_0 \in \Kset$
satisfying $\norm{x - x_0} \le \alpha \distK(x)$
for $\alpha \ge 1$.
Consider the following.
\begin{prop}\label{prop:ideal-rel-error}
  If
  $\alpha \ge 1 + \frac{2 \nu^2}{n} \paren{\frac{1}{e} +
    \log\pfrac{\abs{\Kset}}{\eta}}$ and
  $\frac{1}{\nu}\distK(x) \le \sqrt{n} \sigma$, then
  \begin{align*}
    \norm{x - \sigma \eps^*(x, \sigma) - \projK(x)} \le \eta \distK(x) + C_{x, \alpha},
  \end{align*}
where $C_{x, \alpha} := \sup_{x_0 \in N_{\alpha}}\| x_0-\projK(x)\|$.
\end{prop}

We can guarantee $C_{x, \alpha}$ is zero when $\distK(x)$ is small compared to
the minimum pairwise distance of points in $\Kset$.  Combined with
\Cref{prop:rel-error-large-noise}, this shows the ideal denoiser has low
relative error in both the large and small noise regimes.  We cannot guarantee
that $C_{x, \alpha}$ is small relative to $\distK(x)$ in all regimes, however,
due to pathological cases, e.g., $x$ is exactly between two points in
$N_{\alpha}$.  Nevertheless, \Cref{fig:ideal-denoiser} empirically shows that
the relative error of the ideal denoiser
$\norm{x_t - \sigma \eps^*(x, \sigma_t) - \projK(x_t)}/\distK(x_t)$ is small at
\emph{all pairs} $(x_t, \sigma_t)$ generated by the DDIM sampler, suggesting
these pathologies do not appear in practice.

\section{Gradient Descent Analysis of Sampling}\label{sec:sampling}
Having justified the relative error model in \Cref{sec:denoiseVsProj}, we now
use it to study the DDIM sampler. As a warmup, we first consider the limiting
case of zero error, where we see DDIM is precisely gradient descent on the
squared-distance function with step-size determined by $\sigma_t$. We then
generalize this result to arbitrary $(\eta, \nu)$, showing DDIM is equivalent to
gradient-descent with relative error. Proofs are postponed to
\Cref{sec:proj-proofs}.

\subsection{Warmup: Exact Projection and Gradient Descent} \label{sec:exact-projection-gd}

We state our zero-error assumption in terms of the error-model as follows.
\begin{ass}\label{ass:denoiseIsProjection}
$\epsilon_{\theta}$ is a $(0, 1)$-approximate projection.
\end{ass}

We can now characterize DDIM as follows.
\begin{thm}\label{thm:mainDDIM}
  Let $x_N, \ldots, x_0$ denote a sequence~\eqref{eq:DDIM-scaled} generated by
  DDIM on a schedule $\braces{\sigma_t}_{t=N}^0$ and
  $f(x):= \frac{1}{2}\distK(x)^2$.  Suppose that \Cref{ass:denoiseIsProjection}
  holds, $\nabla f(x_t)$ exists for all $t$ and
  $\distK(x_N) = \sqrt n \sigma_N$. Then $x_t$ is a sequence generated by
  gradient descent with step-size $\beta_t: = 1 - \sigma_{t-1}/\sigma_t$:
  \begin{align*}
    x_{t-1} = x_t -  \beta_t \nabla f(x_t),
  \end{align*}
  and $\distK(x_t) = \sqrt{n}\sigma_t$ for all $t$.
\end{thm}
We remark that the existence of $\nabla f(x_t)$ is a weak assumption as it is
generically satisfied by almost all $x \in \R^n$.

\subsection{Approximate Projection and Gradient Descent with Error}\label{sec:samplingInexact}

We next establish upper and lower-bounds of distance under approximate gradient
descent iterations. Given $\braces{\sigma_t}_{t=N}^0$, let
$\beta_{t} := 1 - \sigma_{t-1}/\sigma_t$ and
\begin{align*}
  L_t^{\sigma, \eta} :=\prod^{N}_{i=t} (1- \beta_i(\eta+1)),
  U_t^{\sigma, \eta} :=\prod^{N}_{i=t} (1+ \beta_i(\eta-1)).
\end{align*}

\begin{lem}~\label{lem:gradErrorBounds}
For $\coneName \subseteq \mathbb{R}^n$,
let $f(x) := \frac{1}{2} \distK(x)^2$. If $x_{t-1} = x_t - \beta_t (\nabla f(x_t) + e_t)$
for $e_t$ satisfying $\|e_t\| \le \eta \distK(x_t)$ and $0 \le \beta_t \le 1$, then
\begin{align*}
  L_t^{\sigma, \eta} \distK(x_N) \le  \distK(x_{t-1}) \le U_t^{\sigma, \eta} \distK(x_N).
\end{align*}
\end{lem}
Observe that the distance upper bound decreases
only if $\beta_i < 1$ when $\eta > 0$.  This conforms with our intuition that step sizes are limited
by the error in our gradient estimates.

The challenge in applying \Cref{lem:gradErrorBounds} to DDIM lies in the
specifics of our relative error model, which states that
$\epsilon_{\theta}(x_t, \sigma_t)$ provides an $\eta$-accurate estimate of
$\nabla f(x_t)$ \emph{only if} $\sigma_t$ provides a $\nu$-accurate estimate of
$\distK(x_t)$.  Hence we must first control the difference between $\sigma_t$
and distance $\distK(x_t)$ by imposing the following conditions on
$\sigma_t$.
\begin{defn}
We say that parameters $\braces{\sigma_t}_{t=N}^0$ are $(\eta, \nu)$-admissible
if, for all $t \in \{1, \ldots, N\}$,
\begin{align}\label{eq:DDIMParamCond}
  \textstyle
  \frac{1}{\nu}   U_t^{\sigma, \eta}
  \le \prod^N_{i=t} (1-\beta_i)
  \le \nu  L_t^{\sigma, \eta}.
\end{align}
\end{defn}
Intuitively, an admissible schedule decreases $\sigma_t$ slow enough
(corresponding to taking smaller gradient steps) to ensure
$\frac{1}{\nu} \distK(x_t) \le \sqrt{n} \sigma_t \le \nu \distK(x_t)$ holds at
each iteration. Our analysis assumes admissibility of the noise
schedule and our relative-error model (\Cref{def:main}):
\begin{ass}\label{ass:main}
  For $\eta > 0$ and $\nu \ge 1$,
   $\braces{\sigma_t}_{t=N}^0$ is $(\eta, \nu)$-admissible
   and $\epsilon_{\theta}$ is an $(\eta, \nu)$-approximate projection.
\end{ass}
Our main result follows. In simple terms, it states  that DDIM is approximate
gradient descent, admissible schedules $\sigma_t$ are good
estimates of distance, and the error bounds of \Cref{lem:gradErrorBounds} hold.
\begin{thm}[DDIM with relative error]\label{thm:DDIMWithError}
  Let $x_t$ denote the sequence generated by DDIM. Suppose \Cref{ass:main}
  holds, the gradient of $f(x):= \frac{1}{2}\distK(x)^2$ exists for all $x_t$
  and $\distK(x_N) = \sqrt n \sigma_N$. Then:
  \begin{itemize}
    \item $x_t$ is generated by
      approximate gradient descent iterations of the form  in \Cref{lem:gradErrorBounds} with $\beta_t = 1 - \sigma_{t-1}/\sigma_{t}$.
   \item $\frac{1}{\nu} \distK(x_{t})  \le \sqrt{n} \sigma_{t} \le \nu \distK(x_{t})$
   for all $t$.
   \item  $\distK(x_N) L_t^{\sigma, \eta}  \le \distK(x_{t-1}) \le \distK(x_N)  U_t^{\sigma, \eta}$
   \end{itemize}
\end{thm}

\subsection{Admissible Log-Linear Schedules for DDIM}~\label{sec:admissible-params}
We next characterize admissible $\sigma_t$ of
the form $\sigma_{t-1} = (1-\beta) \sigma_t$
where $\beta$ denotes a \emph{constant} step-size.
This illustrates  that admissible $\sigma_t$-sequences
not only \emph{exist}, they can also
be explicitly constructed from $(\eta, \nu)$.

\begin{thm}\label{thm:DDIMStepSize}
Fix $\beta \in \mathbb{R}$ satisfying $0 \le \beta < 1$
and suppose that  $\sigma_{t-1} = (1-\beta) \sigma_t$.
Then $\sigma_t$ is $(\eta, \nu)$-admissible
if and only if $\beta \le  \beta_{*, N}$ where $\beta_{*, N}:=\frac{c}{\eta+c}$ for $c := 1-\nu^{-1/N}$.
\end{thm}
Suppose we fix $(\eta, \nu)$ and choose, for a given $N$, the step-size
$\beta_{*, N}$. It is natural to ask how the
error bounds of \Cref{thm:DDIMWithError} change as
$N$ increases.  We establish
the limiting behavior of the \emph{final} output $(\sigma_0, x_0)$
of DDIM\@.
\begin{thm}\label{thm:DDIMStepSizeLimit}
Let $x_N, \ldots, x_1, x_0$ denote the sequence generated by DDIM
with $\sigma_t$ satisfying $\sigma_{t-1} =  (1-\beta_{*, N})\sigma_t$
for $\nu \ge 1$ and $\eta > 0$.
Then
\begin{itemize}
 \item $\Lim{N\rightarrow \infty} \frac{\sigma_0}{\sigma_N} = \Lim{N\rightarrow \infty}(1-\beta_{*,N})^N = \nu^{-1/\eta}$.
 \item $\Lim{N\rightarrow \infty} \frac{\distK(x_0)}{\distK(x_N)} \le \Lim{N\rightarrow \infty}(1+ (\eta-1)\beta_{*,N})^N = \nu^{\frac{\eta-1}{\eta}}$.
\end{itemize}
\end{thm}
This theorem illustrates that final error, while bounded, need not converge
to zero under our error model. This motivates
heuristically updating the step-size from $\beta_{*, N}$ to a full step $(\beta = 1)$
during the final DDIM iteration. We adopt this approach in our
experiments~(\Cref{sec:experiments}).

Next we demonstrate an explicit construction of an admissible schedule using
numerical estimates of the error parameters on an image dataset.

\begin{ex}[Construction of admissible schedule]
  Let the CIFAR-10 training set be $\mathcal{K}$ and the ideal denoiser be
  $\epsilon_\theta$. From \Cref{fig:ideal-denoiser}, which plots the relative
  projection error relative to the training set, we see that $\eta \le 0.1$.
  Our experiments comparing $\distK(x_t)$ with $\sqrt{n} \sigma_t$ suggest that
  $\nu = 2$ is a conservative estimate, as the error in $\distK(x_t)$ is bounded
  by this amount throughout the sampling trajectories. Theorem 4.3 shows that if
  $\sigma_{t-1}/\sigma_t \ge \frac{\eta}{\eta+1-\nu^{-1/N}}$, then
  $\sigma_0, \ldots, \sigma_N$ is an admissible schedule. With $\eta=0.1$,
  $\nu=2$ and $N=50$, we obtain $\sigma_{t-1}/\sigma_t \ge 0.88$. This is very
  close to the value of $\sigma_{t-1}/\sigma_t = 0.85$ in the schedule used in
  our sampler in \Cref{sec:exp-noise-schedule}.
\end{ex}

\confonly{
\begin{figure*}
  \centering
  \begin{tabular}{ccccc}
    \textbf{Ours}
    & \textbf{UniPC}
    & \textbf{DPM++}
    & \textbf{PNDM}
    & \textbf{DDIM} \\
    & \tiny{\citep{zhao2023unipc}} & \tiny{\citep{lu2022dpmpp}} & \tiny{\citep{liu2022pseudo}} & \tiny{\citep{song2020denoising}} \\
    \textbf{FID 13.77} & 15.59 & 15.43 & 19.43 & 14.06 \\[0.2cm]
    \img{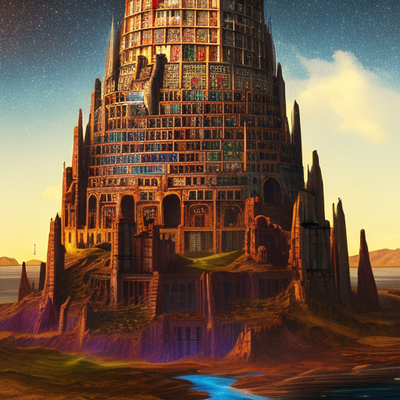} & \img{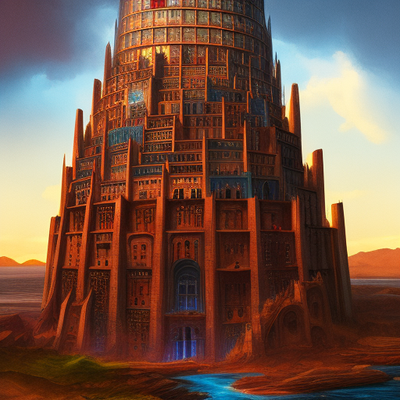} &
    \img{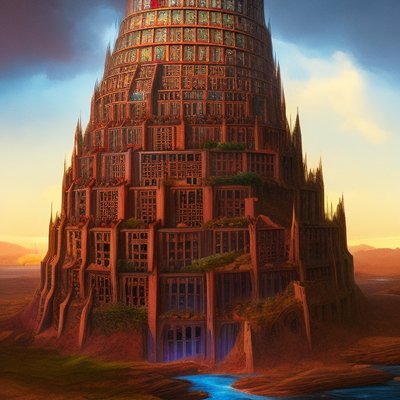} & \img{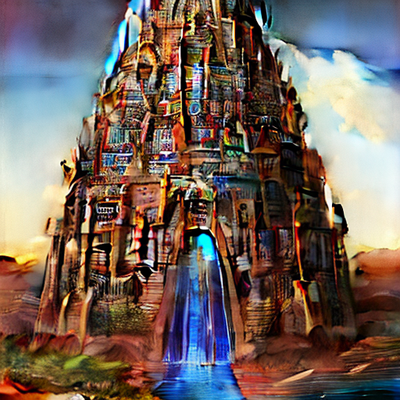} & \img{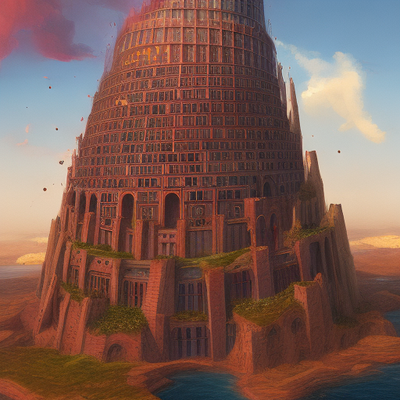} \\
    \img{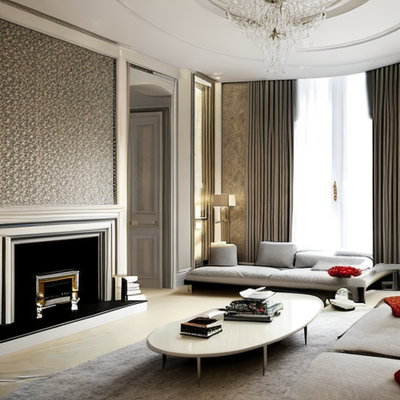} & \img{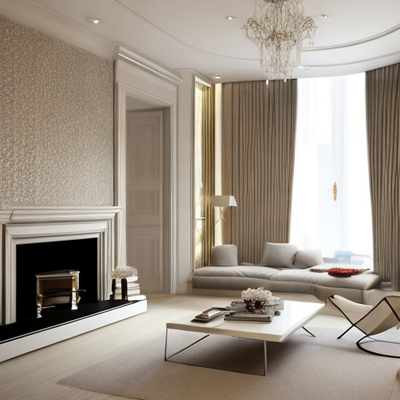} &
    \img{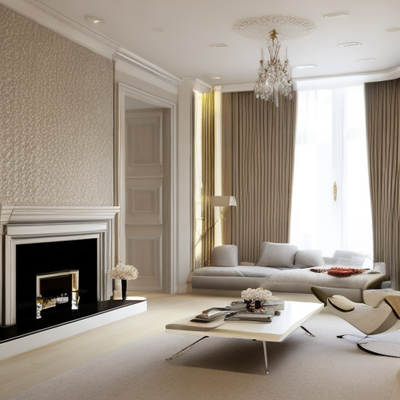} & \img{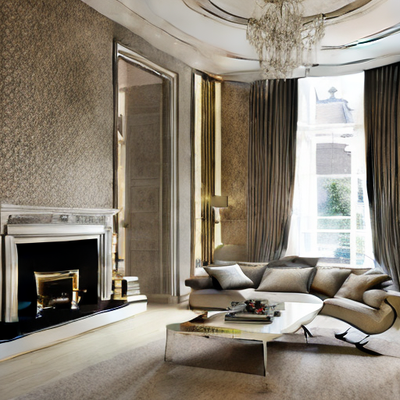} & \img{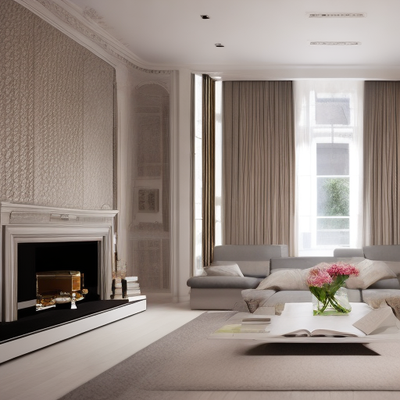} \\
    \img{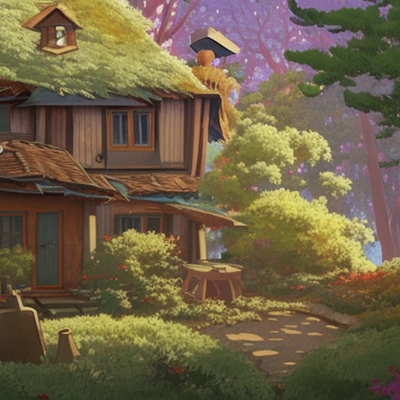} & \img{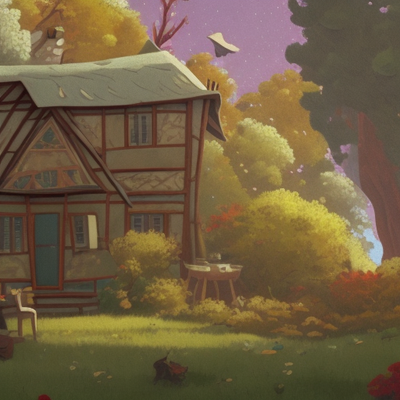} &
    \img{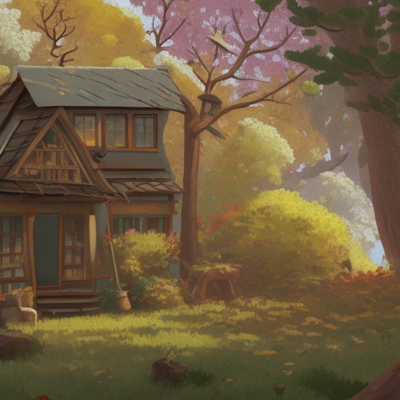} & \img{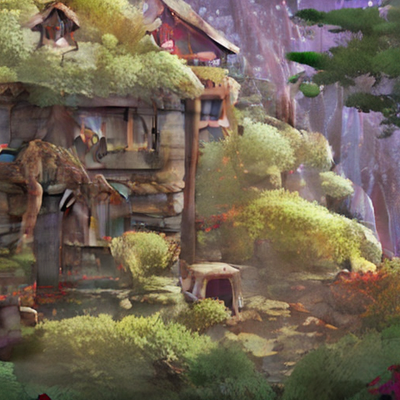} & \img{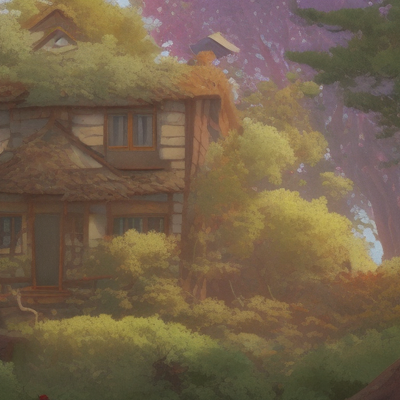}
  \end{tabular}
  \caption{Outputs of our gradient-estimation sampler on text-to-image Stable Diffusion compared to
    other commonly used samplers, when limited to $N = 10$ function
    evaluations. We also report FID scores for text-to-image generation on
    MS-COCO 30K.}
  \label{fig:sd-figures}
\end{figure*}
}

\arxivonly{
  \begin{figure}
  \centering
  \begin{minipage}[t]{.48\linewidth}
    \begin{algorithm}[H]
      \caption{DDIM sampler \citep{song2020denoising}}
      \label{alg:ddim}
      \begin{algorithmic}
        \Require $(\sigma_N, \ldots, \sigma_0)$, $x_N \sim \mathcal{N}(0, I)$, $\epsilon_\theta$
        \Ensure Compute $x_0$ with $N$ evaluations of $\epsilon_\theta$
        \For{$t = N, \ldots, 1$}
        \State $x_{t-1} \gets x_t + (\sigma_{t-1} - \sigma_t) \epsilon_\theta(x_t, \sigma_t)$
        \EndFor
        \State \Return $x_0$
      \end{algorithmic}
    \end{algorithm}
  \end{minipage}\,\,\,\,
  \begin{minipage}[t]{.48\linewidth}
    \begin{algorithm}[H]
      \caption{Our gradient-estimation sampler}
      \label{alg:second-order}
      \begin{algorithmic}
        \Require $(\sigma_N, \ldots, \sigma_0)$, $x_N \sim \mathcal{N}(0, I)$, $\epsilon_\theta$
        \Ensure Compute $x_0$ with $N$ evaluations of $\epsilon_\theta$
        \State $x_{N-1} \gets x_N + (\sigma_{N-1} - \sigma_N)\epsilon_\theta(x_N, \sigma_N)$
        \For{$t = N-1, \ldots, 1$}
        \State $\bar{\epsilon}_t \gets 2\epsilon_\theta(x_t, \sigma_t) - \epsilon_\theta(x_{t+1}, \sigma_{t+1})$
        \State $x_{t-1} \gets x_t + (\sigma_{t-1} - \sigma_t) \bar{\epsilon}_t$
        \EndFor
        \State \Return $x_0$
      \end{algorithmic}
  \end{algorithm}
  \end{minipage}

\end{figure}
}

\confonly{
\begin{algorithm}[H]
  \caption{DDIM sampler \citep{song2020denoising}}
  \label{alg:ddim}
  \begin{algorithmic}
    \Require $(\sigma_N, \ldots, \sigma_0)$, $x_N \sim \mathcal{N}(0, I)$, $\epsilon_\theta$
    \Ensure Compute $x_0$ with $N$ evaluations of $\epsilon_\theta$
    \For{$t = N, \ldots, 1$}
    \State $x_{t-1} \gets x_t + (\sigma_{t-1} - \sigma_t) \epsilon_\theta(x_t, \sigma_t)$
    \EndFor
    \State \Return $x_0$
  \end{algorithmic}
\end{algorithm}

\begin{algorithm}[H]
  \caption{Our gradient-estimation sampler}
  \label{alg:second-order}
  \begin{algorithmic}
    \Require $(\sigma_N, \ldots, \sigma_0)$, $x_N \sim \mathcal{N}(0, I)$, $\epsilon_\theta$
    \Ensure Compute $x_0$ with $N$ evaluations of $\epsilon_\theta$
    \State $x_{N-1} \gets x_N + (\sigma_{N-1} - \sigma_N)\epsilon_\theta(x_N, \sigma_N)$
    \For{$t = N-1, \ldots, 1$}
    \State $\bar{\epsilon}_t \gets 2\epsilon_\theta(x_t, \sigma_t) - \epsilon_\theta(x_{t+1}, \sigma_{t+1})$
    \State $x_{t-1} \gets x_t + (\sigma_{t-1} - \sigma_t) \bar{\epsilon}_t$
    \EndFor
    \State \Return $x_0$
  \end{algorithmic}
\end{algorithm}
}

\section{Improving Deterministic Sampling Algorithms via Gradient Estimation}\label{sec:newAlgorithms}

\Cref{sec:denoiseVsProj} establishes that $\epsilon_{\theta}(x, \sigma)
\approx \sqrt{n} \nabla \distK(x)$ when $\distK(x) \approx \sqrt{n} \sigma$.
We next exploit  an invariant property of $\nabla \distK(x)$ to reduce the
prediction error of $\epsilon_{\theta}$ via \emph{gradient estimation}.

\arxivonlyverbatim{
\begin{wrapfigure}{r}{0.4\textwidth}
  \centering
  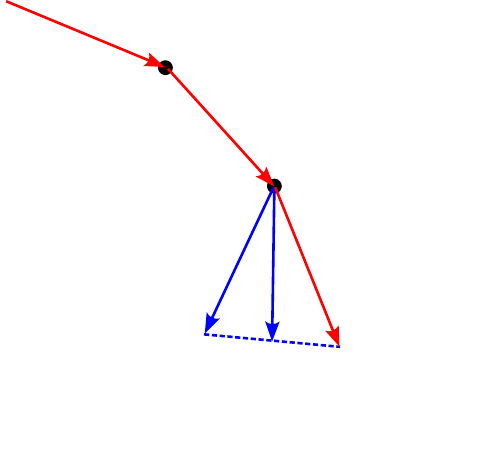
  \caption{Illustration of our choice of $\bar{\epsilon}_t$}
  \label{fig:denoiser-step}
\end{wrapfigure}
}

The gradient $\nabla \distK(x)$ is \emph{invariant} along line segments between
a point $x$ and its projection $\projK(x)$, i.e., letting $\hat x = \projK(x)$, for all $\theta \in (0, 1]$ we have
\begin{align}\label{eq:distConst}
\nabla \distK(\theta x +(1-\theta) \hat x) = \nabla \distK(x). %
\end{align}
Hence, $\epsilon_{\theta}(x, \sigma)$ should be
(approximately) constant on this line-segment
under our assumption that $\epsilon_{\theta}(x, \sigma) \approx \sqrt{n} \nabla \distK(x)$
when $\distK(x) \approx \sqrt{n} \sigma$.
 Precisely, for $x_1$ and $x_2$ on this line-segment,
we should have
\begin{align}\label{eq:constantGrad}
\epsilon_{\theta}(x_1, \sigma_{t_1}) \approx \epsilon_{\theta}(x_2, \sigma_{t_2})
\end{align}
if $t_i$ satisfies $\distK(x_i) \approx \sqrt{n} \sigma_{t_i}$.  This property
suggests combining previous denoiser outputs
$\{\epsilon_{\theta}(x_i, \sigma_i)\}^N_{i=t+1}$ to estimate
$\epsilon_t := \sqrt n \nabla \distK(x_t)$.  We next propose a practical
\emph{second-order method} \footnote{This method is second-order in the sense
  that the update step uses previous values of $\epsilon_\theta$, and should not
  be confused with second-order derivatives.}  for this estimation that combines
the current denoiser output with the previous.  Recently introduced
\emph{consistency models} \citep{song2023consistency} penalize violation
of~\eqref{eq:constantGrad} during \emph{training}.  Interpreting denoiser output
as $\nabla \distK(x)$ and invoking~\eqref{eq:distConst} offers an alternative
justification for these models.
\confonly{
\begin{figure}[h]
  \centering
  \input{figures/drawing.pdf_tex}
  \caption{Illustration of our choice of $\bar{\epsilon}_t$}
  \vspace{-0.5cm}
  \label{fig:denoiser-step}
\end{figure}
}

Let $e_t(\epsilon_t) = \epsilon_t - \epsilon_{\theta}(x_t, \sigma_t)$ be the error
of $\epsilon_{\theta}(x_t, \sigma_t)$ when predicting $\epsilon_t$. To estimate $\epsilon_t$
from $\epsilon_{\theta}(x_t, \sigma_t)$, we minimize the norm of this error concatenated
over two time-steps.
Precisely, letting $y_t(\epsilon) = (e_t(\epsilon), e_{t+1}(\epsilon))$,
we compute
\begin{align}\label{eq:estimatorGeneralForm}
  \bar{\epsilon}_t := \argmin_{\epsilon} \norm{y_t(\epsilon)}^2_W,
\end{align}
where $W$ is a specified positive-definite weighting matrix. In
\Cref{sec:second-order-weighting} we show that this error model, for a particular family of weight
matrices, results in the update rule
\begin{align} \label{eq:eps-bar-gamma}
  \bar\epsilon_t = \gamma \epsilon_{\theta}(x_t, \sigma_t) + (1-\gamma) \epsilon_{\theta}(x_{t+1}, \sigma_{t+1}),
\end{align}
where we can search over $W$ by searching over $\gamma$.

\section{Experiments}\label{sec:experiments}

\confonly{
\begin{figure}[H]
  \centering
  \includegraphics[width=0.3\textwidth]{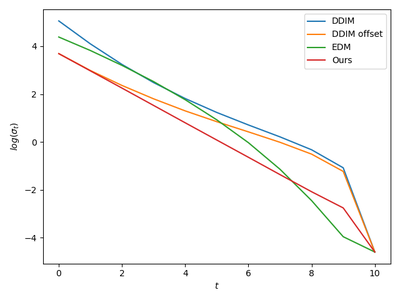}
  \caption{Plot of different choices of $\log(\sigma_t)$ for $N=10$.}
  \label{fig:sigma_schedule}
\end{figure}
\vspace{-.5cm}
\begin{table}[H]
  \centering
  \begin{tabular}[t]{lll}
    \toprule
    Schedule    & CIFAR-10  & CelebA \\
    \midrule
    DDIM        & 16.86 & 18.08 \\
    DDIM Offset & 14.18 & 15.38 \\
    EDM         & 20.85 & 16.72 \\
    Ours        & \textbf{13.25} & \textbf{13.55} \\
    \bottomrule
  \end{tabular}
  \vspace{0.5cm}
  \caption{FID scores of the DDIM sampler (\Cref{alg:ddim}) with different
    $\sigma_t$ schedules on the CIFAR-10 model for $N=10$ steps.}
  \label{table:schedule-fid}
\end{table}
}

\arxivonly{
\begin{figure*}
  \centering
  \begin{tabular}{ccccc}
    \textbf{Ours}
    & \textbf{UniPC}
    & \textbf{DPM++}
    & \textbf{PNDM}
    & \textbf{DDIM} \\
    & \tiny{\citep{zhao2023unipc}} & \tiny{\citep{lu2022dpmpp}} & \tiny{\citep{liu2022pseudo}} & \tiny{\citep{song2020denoising}} \\
    \textbf{FID 13.77} & 15.59 & 15.43 & 19.43 & 14.06 \\[0.2cm]
    \img{figures/sd_1_2nd} & \img{figures/sd_1_unipc} &
    \img{figures/sd_1_dpm} & \img{figures/sd_1_pndm} & \img{figures/sd_1_ddim} \\
    \img{figures/sd_2_2nd} & \img{figures/sd_2_unipc} &
    \img{figures/sd_2_dpm} & \img{figures/sd_2_pndm} & \img{figures/sd_2_ddim} \\
    \img{figures/sd_3_2nd} & \img{figures/sd_3_unipc} &
    \img{figures/sd_3_dpm} & \img{figures/sd_3_pndm} & \img{figures/sd_3_ddim}
  \end{tabular}
  \caption{Outputs of our gradient-estimation sampler on text-to-image Stable Diffusion compared to
    other commonly used samplers, when limited to $N = 10$ function
    evaluations. We also report FID scores for text-to-image generation on
    MS-COCO 30K.}
  \label{fig:sd-figures}
\end{figure*}
}

We evaluate modifications of DDIM (\Cref{alg:ddim}) that leverage insights from
\Cref{sec:newAlgorithms} and \Cref{sec:admissible-params}.  Following
\Cref{sec:newAlgorithms} we modify DDIM to use a second-order update that
corrects for error in the denoiser output (\Cref{alg:second-order}).
Specifically, we use the~\Cref{eq:eps-bar-gamma} update with an empirically
tuned $\gamma$. We found that setting $\gamma=2$ works well for $N < 20$; for
larger $N$ slightly increasing $\gamma$ also improves sample quality (see
\Cref{sec:futher-experiments} for more details). A comparison of this update
with DDIM is visualized in \Cref{fig:denoiser-step}. Following
\Cref{sec:admissible-params}, we select a noise schedule
$(\sigma_N, \ldots, \sigma_0)$ that decreases at a log-linear (geometric)
rate. The specific rate is determined by an initial and target noise level.  Our
$\sigma_t$ schedule is illustrated in \Cref{fig:sigma_schedule}, along with
other commonly used schedules.  We note that log-linear schedules have been
previously proposed for SDE-samplers \citep{song2020score}; to our knowledge we
are the first to propose and analyze their use for DDIM\footnote{DDIM is usually
  presented using not $\sigma_t$ but parameters $\alpha_t$ satisfying
  $\sigma^2_t = (1-\alpha_t)/\alpha_t$.  Linear updates of $\sigma_t$ are less
  natural when expressed in terms of $\alpha_t$.}. All the experiments were run
on a single Nvidia RTX 4090 GPU. Code for the experiments is available at
\url{https://github.com/ToyotaResearchInstitute/gradient-estimation-sampler}

\subsection{Evaluation of Noise Schedule}\label{sec:exp-noise-schedule}

\arxivonly{
  \begin{figure}[t]
  \centering
  \begin{minipage}[b]{.45\linewidth}
    \begin{figure}[H]
      \centering
      \includegraphics[width=0.8\textwidth]{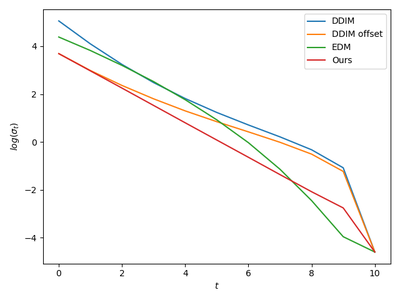}
      \caption{Plot of different choices of $\log(\sigma_t)$ for $N=10$.}
      \label{fig:sigma_schedule}
    \end{figure}
  \end{minipage}\,\,\,\,
  \begin{minipage}[b]{.45\linewidth}
    \begin{table}[H]
      \centering
      \begin{tabular}[t]{lll}
        \toprule
        Schedule    & CIFAR-10  & CelebA \\
        \midrule
        DDIM        & 16.86 & 18.08 \\
        DDIM Offset & 14.18 & 15.38 \\
        EDM         & 20.85 & 16.72 \\
        Ours        & \textbf{13.25} & \textbf{13.55} \\
        \bottomrule
      \end{tabular}
      \vspace{0.5cm}
      \caption{FID scores of the DDIM sampler (\Cref{alg:ddim}) with different
        $\sigma_t$ schedules on the CIFAR-10 model for $N=10$ steps.}
      \label{table:schedule-fid}
    \end{table}
  \end{minipage}
\end{figure}
}
In \Cref{fig:sigma_schedule} we plot our schedule (with our choices
of $\sigma_t$ detailed in \Cref{sec:exp-details}) with three other commonly
used schedules on a log scale. The first is the evenly spaced
subsampling of the training noise levels used by DDIM. The second ``DDIM
Offset'' uses the same even spacing but starts at a smaller $\sigma_N$, the same
as that in our schedule. This type of schedule is typically used for guided
image generation such as SDEdit \citep{meng2021sdedit}. The third ``EDM'' is the
schedule used in \citet[Eq. 5]{karras2022elucidating}, with
$\sigma_{\max} = 80, \sigma_{\min}=0.002$ and $\rho=7$.

We then test these schedules on the DDIM sampler \Cref{alg:ddim} by sampling
images with $N=10$ steps from the CIFAR-10 and CelebA models. We see that in
\Cref{table:schedule-fid} that our schedule improves the FID of the DDIM sampler
on both datasets even without the second-order updates. This is in part due to
choosing a smaller $\sigma_N$ so the small number of steps can be better spent
on lower noise levels (the difference between ``DDIM'' and ``DDIM Offset''), and
also because our schedule decreases $\sigma_t$ at a faster rate than DDIM (the
difference between ``DDIM Offset'' and ``Ours'').

\subsection{Evaluation of Full Sampler}
\begin{table*}
  \centering
  \setlength\tabcolsep{3pt}
  \small
  \confonly{
    \caption{FID scores of our gradient-estimation sampler compared to that of
      other samplers for pretrained CIFAR-10 and CelebA models with a discrete
      linear schedule. The first half of the table shows our computational
      results whereas the second half of the table show results taken from the
      respective papers. *Results for $N=25$}
  }
  \label{table:sampler-fid}
  \begin{tabular}[t]{lllllllll}
    \toprule
    & \multicolumn{4}{c}{CIFAR-10 FID} & \multicolumn{4}{c}{CelebA FID} \\
    Sampler & $N=5$ & $N=10$ & $N=20$ & $N=50$ & $N=5$ & $N=10$ & $N=20$ & $N=50$ \\
    \midrule
    Ours & \textbf{12.57} & \textbf{3.79} & \textbf{3.32} & \textbf{3.41} & \textbf{10.76} & \textbf{4.41} & 3.19 & 3.04 \\
    DDIM \tiny{\citep{song2020denoising}} & 47.20 & 16.86 & 8.28 & 4.81 & 32.21 & 18.08 & 11.81 & 7.39  \\
    \midrule
    PNDM \tiny{\citep{liu2022pseudo}} & 13.9 & 7.03 & 5.00 & 3.95 & 11.3 & 7.71 & 5.51 & 3.34  \\
    DPM \tiny{\citep{lu2022dpm}} &  & 6.37 & 3.72 & 3.48 &  & 5.83 & \textbf{2.82} & \textbf{2.71}  \\
    DEIS \tiny{\citep{zhang2022fast}} & 18.43 & 7.12 & 4.53 & 3.78 & 25.07 & 6.95 & 3.41 & 2.95 \\
    UniPC \tiny{\citep{zhao2023unipc}} & 23.22 & 3.87 &   &   &   &   &   &   \\
    A-DDIM \tiny{\citep{bao2022analytic}} &   & 14.00 & 5.81* & 4.04 &   & 15.62 & 9.22* & 6.13 \\
    \bottomrule
  \end{tabular}
  \arxivonly{
    \caption{FID scores of our gradient-estimation sampler compared to that of
      other samplers for pretrained CIFAR-10 and CelebA models with a discrete
      linear schedule. The first half of the table shows our computational
      results whereas the second half of the table show results taken from the
      respective papers. *Results for $N=25$}
  }
\end{table*}

We quantitatively evaluate our gradient-estimation sampler
(\Cref{alg:second-order}) by computing the Fr\'echet Inception Distance (FID)
\citep{heusel2017gans} between all the training images and 50k generated
images. We use denoisers from \cite{ho2020denoising, song2020denoising} that
were pretrained on the CIFAR-10 (32x32) and CelebA (64x64)
datasets~\citep{krizhevsky2009learning, liu2015faceattributes}. We compare our
results with other samplers using the same denoisers. The FID scores are
tabulated in \Cref{table:sampler-fid}, showing that our sampler achieves better
performance on both CIFAR-10 (for $N=5,10,20,50$) and CelebA (for $N=5,10$).

We also incorporated our sampler into Stable Diffusion (a latent diffusion
model). We change the noise schedule $\sigma_t$ as described in
\Cref{sec:exp-details}. In \Cref{fig:sd-figures}, we show some example results
for text to image generation in $N=10$ function evaluations, as well as FID
results on 30k images generated from text captions drawn the MS COCO
\citep{lin2014microsoft} validation set. From these experiments we can see that
our sampler performs comparably to other commonly used samplers, but with the
advantage of being much simpler to describe and implement.

\begin{figure}
  \centering
  \includegraphics[width=0.5\textwidth]{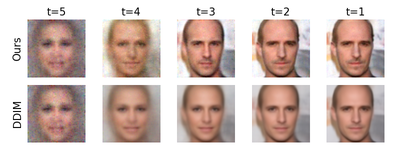}
  \caption{A comparison of our gradient-estimation sampler with DDIM on the CelebA dataset with
    $N=5$ steps. }
  \label{fig:sampler_comparison}
  \vspace{-0.3cm}
\end{figure}

\section{Related Work and Discussion}\label{sec:discussion}
\paragraph{Learning diffusion models}
Diffusion was originally introduced as a
variational inference method that learns to reverse a noising process~\cite{sohl2015deep}. This approach
was empirically improved by \citep{ho2020denoising,nichol2021improved} by
introducing the training loss~\eqref{eq:trainingLoss}, which is different from the original variational lower
bound. This improvement is justified from the perspective of denoising score
matching \citep{song2019generative,song2020score}, where the $\epsilon_\theta$
is interpreted as $\nabla \log(p(x_t, \sigma_t))$, the gradient of the log
density of the data distribution perturbed by noise. Score matching is also
shown to be equivalent to denoising autoencoders with Gaussian noise
\citep{vincent2011connection}.

\paragraph{Sampling from diffusion models}
Samplers for diffusion models started with probabilistic methods
(e.g. \cite{ho2020denoising}) that formed the reverse process by conditioning on
the denoiser output at each step. In parallel, score based models
\citep{song2019generative,song2020score} interpret the forward noising process
as a stochastic differential equation (SDE), so SDE solvers based on Langevian
dynamics \citep{welling2011bayesian} are employed to reverse this process. As
models get larger, computational constraints motivated the development of more
efficient samplers. \cite{song2020denoising} then discovered that for smaller
number of sampling steps, deterministic samplers perform better than stochastic
ones. These deterministic samplers are constructed by reversing a non-Markovian
process that leads to the same training objective, which is equivalent to
turning the SDE into an ordinary differential equation (ODE) that matches its
marginals at each sampling step.

This led to a large body of work focused on developing ODE and SDE solvers for
fast sampling of diffusion models, a few of which we have evaluated in
\Cref{table:sampler-fid}. Most notably, \cite{karras2022elucidating} put
existing samplers into a common framework and isolated components that can be
independently improved. Our gradient-estimation sampler \Cref{alg:second-order}
bears most similarity to linear multistep methods, which can also be interpreted
as accelerated gradient descent \citep{scieur2017integration}. What differs is
the error model: ODE solvers aim to minimize discretization error whereas we aim
to minimize gradient estimation error, resulting in different ``optimal''
samplers.

\paragraph{Linear-inverse problems and conditioning}
Several authors \citep{kadkhodaie2020solving, chung2022improving,
  kawar2022denoising} have devised samplers for finding images that satisfy
linear equations $Ax=b$.  Such linear inverse problems generalize inpainting,
colorization, and compressed sensing.  In our framework, we can interpret this
samplers as algorithms for equality constraint minimization of the distance
function, a classical problem in optimization. Similarly, the widely used
technique of \emph{conditioning} \citep{dhariwal2021diffusion} can be interpreted
as multi-objective optimization, where minimization of distance is replaced with
minimization of $\distK(x)^2 + g(x)$ for an auxiliary objective function $g(x)$.

\paragraph{Score distillation sampling}
We illustrate the potential of our framework for discovering new applications of
diffusion models by deriving Score Distillation Sampling (SDS), a method for
parameterized optimization introduced in~\cite{poole2022dreamfusion} in the
context of text to 3D object generation. At a high-level, this technique finds
$(x, \theta)$ satisfying non-linear equations $x = g(\theta)$ subject to the
constraint $x \in \coneName$, where $\coneName$ denotes the image manifold. It
does this by iteratively updating $x$ with a direction proportional to
$(\epsilon_\theta(x+\sigma \epsilon, \sigma) - \epsilon)\nabla g(\theta)$, where
$\sigma$ is a randomly chosen noise level and $\epsilon \sim \mathcal{N}(0,
I)$. Under our framework, this iteration can be interpreted as gradient descent
on the squared-distance function with gradient
$\frac{1}{2}\nabla_\theta \distK(g(\theta))^2=(x-\projK(x))\nabla g(\theta)$,
with the assumption that $\projK(x) \approx \projK(x + \sigma \epsilon)$, along
with our \Cref{sec:denoiseVsProj} denoising approximation
$\projK(x + \sigma \epsilon) \approx x + \sigma\epsilon -
\sigma\epsilon_\theta(x + \sigma \epsilon, \sigma)$.

\paragraph{Flow matching}
Flow matching \cite{lipman2022flow} offers a different interpretation and
generalization of diffusion models and deterministic sampling. Under this
interpretation, the learned $\epsilon_\theta$ represents a time-varying vector
field, defining probability paths that transport the initial Gaussian
distribution to the data distribution. For $\epsilon_\theta$ learned with the
denoising objective, we can interpret this vector field as the gradient of the
smoothed squared-distance function $\distop^2_\Kset(x, \sigma)$ (where $\sigma$
changes as a function of $t$), thus moving along a probability path in this
vector field minimizes the distance to the manifold.

\paragraph{Learning the distance function}
Reinterpreting denoising as projection, or equivalently
gradient descent on the distance function, has a few immediate implications.
First, it suggests generalizations that draw upon the
literature for computing distance functions and projection operators.
Such techniques include Fast
Marching Methods \citep{sethian1996fast},  kd-trees,
and neural-network approaches, e.g.,~\cite{park2019deepsdf, rick2017one}.
Using concentration inequalities, we can
also interpret training a denoiser as learning
a solution to the \emph{Eikonal PDE},
given by $\|\nabla d(x)\| = 1$.  Other techniques for solving this PDE with
deep neural nets include~\cite{smith2020eikonet,lichtenstein2019deep,bin2021pinneik}.

\section{Conclusion and Future Work} \label{sec:futurework}

We have presented a simple framework for analyzing and generalizing
diffusion models that has led to a new sampling approach and new interpretations
of pre-existing techniques. Moreover, the key objects in our analysis ---the
distance function and the projection operator---are canonical objects in
constrained optimization. We believe our work can lead to new
generative models that incorporate sophisticated objectives and constraints
for a variety of applications. We also believe this work can be leveraged to
incorporate existing denoisers into optimization algorithms in a plug-in-play
fashion, much like the work in \cite{chan2016plug, le2023preconditioned, rick2017one}.

Combining the multi-level noise paradigm of diffusion with distance function
learning~\citep{park2019deepsdf} is an interesting direction, as are
diffusion-models that carry out projection using analytic formulae or simple
optimization routines.

{\small
\bibliographystyle{acm}
\bibliography{../tex_common/bib}
}

\confonly{\newpage}
\appendix
\confonly{\onecolumn}

\section{Equivalent Definitions of DDIM and DDPM} \label{sec:scaled}

The DDPM and DDIM samplers are usually described
in a different coordinate system $z_t$
defined by parameters $\bar \alpha_t$ and the following relations
, where the noise model is defined by a schedule $\abt$:
\begin{align}\label{eq:noise-model-scaled}
  y \approx \sqrt{\abt} z + \sqrt{1-\abt} \epsilon,
\end{align}
with the estimate $\zest := \hat z_0(z_t, t)$ given by
\begin{align}\label{eq:z0def}
   \hat z_0(y, t) := \frac{1}{\sqrt{\abt}}(y - \sqrt{1-\abt} \epsilon'_\theta(y, t)).
\end{align}
We have the following conversion identities between the $x$ and $z$ coordinates:
\begin{align} \label{eq:conversion}
  x_0 = z_0 ,\quad x_t = z_t/\sqrt{\abt}
  ,\quad \sigma_t = \sqrt{\frac{1-\abt}{\abt}}
  ,\quad \epsilon_\theta(y, \sigma_t) = \epsilon'_\theta(y/\sqrt{\abt}, t).
\end{align}
While this change-of-coordinates is used in~\citet[Section
4.3]{song2020denoising} and in~\cite{karras2022elucidating}--and hence not
new-- we rigorously prove equivalence of the DDIM and DDPM samplers given in
\Cref{sec:background} with their original definitions.

\paragraph{DDPM}
Given initial $z_N$, the DDPM sampler constructs the sequence
\begin{align}\label{eq:DDPM}
 z_{t-1} =
\frac{\sqrt{\ab{t-1}}( 1 - \alpha_t)}{1-\abt}  \hat z^t_{0}  +
\frac{\sqrt \alpha_t(1-\ab{t-1})}{1-\abt} z_t
+ \sqrt{\frac{1-\ab{t-1}}{1-\abt}(1-\alpha_t)} w_t,
\end{align}
where $\alpha_t := \abt/\ab{t-1}$ and $w_t \sim \mathcal{N}(0, I)$.  This is
interpreted as sampling $z_{t-1}$ from a Gaussian distribution conditioned on
$z_t$ and $\zest$~\citep{ho2020denoising}.
\begin{prop}[DDPM change of coordinates] \label{prop:ddpm-conversion}
  The sampling update \eqref{eq:DDPM-scaled} is equivalent to the update
  \eqref{eq:DDPM} under the change of coordinates \eqref{eq:conversion}.
  \begin{proof}
    First we write \eqref{eq:DDPM-scaled} in terms of $z_t$,
    $\epsilon'_\theta(z_t,t)$ and $w_t$ using \eqref{eq:z0def}:
    \begin{align*}
      z_{t-1} &= \frac{\sqrt{\ab{t-1}}( 1 - \alpha_t)}{\sqrt{\abt}(1-\abt)}
      \paren{z_t - \sqrt{1-\abt} \epsilon'_\theta(z_t, t)}
      + \frac{\sqrt \alpha_t(1-\ab{t-1})}{1-\abt} z_t
                + \sqrt{\frac{1-\ab{t-1}}{1-\abt}(1-\alpha_t)} w_t \\
              &= \frac{z_t}{\sqrt{\alpha_t}}
                + \frac{\alpha_t-1}{\sqrt{\alpha_t (1-\abt))}} \epsilon'_\theta(z_t, t)
                + \sqrt{\frac{1-\ab{t-1}}{1-\abt}(1-\alpha_t)} w_t.
    \end{align*}
    Next we divide both sides by $\sqrt{\ab{t-1}}$ and change $z_t$ and
    $z_{t-1}$ to $x_t$ and $x_{t-1}$:
    \begin{align*}
      x_{t-1} = x_t + \frac{\alpha_t-1}{\sqrt{\abt (1-\abt)}} \epsilon_\theta(x_t, \sigma_t)
      + \sqrt{\frac{1-\ab{t-1}}{\ab{t-1}} \frac{1-\alpha_t}{1-\abt}} w_t.
    \end{align*}
    Now if we define
    \begin{align*}
      \eta &:= \sqrt{\frac{1-\ab{t-1}}{\ab{t-1}} \frac{1-\alpha_t}{1-\abt}}
             = \sigma_{t-1} \sqrt{\frac{1-\abt/\ab{t-1}}{1-\abt}} , \\
      \sigma_{t'} &:= \sqrt{\sigma_{t-1}^2 - \eta^2}
                    = \sigma_{t-1}\sqrt{\frac{\abt(1/\ab{t-1}-1)}{1-\abt}}
                    = \frac{\sigma_{t-1}^2}{\sigma_t},
    \end{align*}
    it remains to check that
    \begin{align*}
      \sigma_{t'} - \sigma_t
      = \frac{\sigma_{t-1}^2 - \sigma_t^2}{\sigma_t}
      = \frac{1/\ab{t-1} - 1/\abt}{\sqrt{1-\abt}/\sqrt{\abt}}
      = \frac{\alpha_t - 1}{\sqrt{\abt(1-\abt)}}.
    \end{align*}
  \end{proof}
\end{prop}

\paragraph{DDIM}
Given initial $z_N$, the DDIM sampler constructs the sequence
\begin{align} \label{eq:DDIM-eps}
  z_{t-1} &= \sqrt{\ab{t-1}} \zest + \sqrt{1-\ab{t-1}} \epsilon'_\theta(z_t, t),
\end{align}
i.e., it estimates $\zest$ from $z_t$ and then constructs $z_{t-1}$ by simply
updating $\ab{t}$ to $\ab{t-1}$.  This sequence can be equivalently expressed in
terms of $\zest$ as
\begin{align}\label{eq:DDIM}
  z_{t-1} &= \sqrt{\ab{t-1}} \zest + \sqrt{\frac{1-\ab{t-1}}{1-\abt}} (z_t - \sqrt{\abt}\zest).
\end{align}

\begin{prop}[DDIM change of coordinates] \label{prop:ddim-conversion}
  The sampling update \eqref{eq:DDIM-scaled} is equivalent to the update
  \eqref{eq:DDIM} under the change of coordinates \eqref{eq:conversion}.
  \begin{proof}
    First we write \eqref{eq:DDIM-eps} in terms of $z_t$ and
    $\epsilon'_\theta(z_t,t)$ using \eqref{eq:z0def}:
    \begin{align*}
      z_{t-1} = \sqrt{\frac{\ab{t-1}}{\abt}} z_t
      + \paren{\sqrt{1-\ab{t-1}} - \sqrt{\frac{\ab{t-1}}{\abt}} \sqrt{1-\abt}}\epsilon'_\theta(z_t, t).
    \end{align*}
    Next we divide both sides by $\sqrt{\ab{t-1}}$ and change $z_t$ and
    $z_{t-1}$ to $x_t$ and $x_{t-1}$:
    \begin{align*}
      x_{t-1} &= x_t + \paren{\sqrt{\frac{1-\ab{t-1}}{\ab{t-1}}}
                - \sqrt{\frac{\ab{t-1}}{1-\abt}}} \epsilon_\theta(x_t, \sigma_t)\\
      &=  x_t + (\sigma_{t-1} - \sigma_t) \epsilon_\theta(x_t, \sigma_t).
    \end{align*}
  \end{proof}
\end{prop}

\section{Formal Comparison of Denoising and Projection}\label{sec:formalDenoising}

\subsection{Proof of \Cref{prop:OracleDenoisingUnderReachInformal}}
First, we state the formal version of
\Cref{prop:OracleDenoisingUnderReachInformal}
\begin{prop}[Oracle denoising] \label{prop:OracleDenoisingUnderReach}
  Fix $\sigma > 0$, $t > 0$ and let $\kappa(t) := \sqrt{(\sqrt{d} + t)^2 + (\sqrt{n-d} + t)^2}$.
   Given $x_0 \in \coneName$ and
  $\epsilon \sim \mathcal{N}(0, I)$, let $x_\sigma = x_0 + \sigma \epsilon$.
  Suppose that $\reach(\coneName) > \sigma \kappa(t)$ and
$\sqrt{n-d}-\sqrt{d} - 2t > 0$.
Then, for an absolute constant $\alpha > 0$,
we have,
with probability at least $(1-\exp(-\alpha t^2))^2$, that
\[
\sigma(\sqrt{n-d} - \sqrt{d} - 2t) \le \dist(x_\sigma) \le  \sigma(\sqrt{n-d} + \sqrt{d} + 2t)
\]
and
\[
 \|\projK(x_\sigma) - x_0\| \le \frac{C(t)  (\sqrt{d}+t)}{ \sqrt{n-d}-\sqrt{d} - 2t} \distK(x_\sigma)
 \]
  where $C(t) := \frac{\reach(\coneName)}{\reach(\coneName) - \sigma \kappa(t)}$.
\end{prop}
Our proof uses local Lipschitz continuity of the projection operator, stated
formally as follows.
\begin{prop}[Theorem 6.2(vi), Chapter 6 of~\cite{delfour2011shapes}]\label{prop:Lipschitz}
Suppose $0 < \reach(\coneName) < \infty$.
Consider $h > 0$ and $x,y \in \mathbb{R}^n$
satisfying $0 < h < \reach(\coneName)$
and $\distK(x) \le h$ and $\distK(y) \le h$.
Then the projection map satisfies $\|\projK(y) - \projK(x)\| \le \frac{\reach(\coneName)}{\reach(\coneName) - h}\|y-x\|$.
\end{prop}
We also use the following concentration inequalities.
\begin{prop}\label{prop:concenIneq}
Let $w \sim \mathcal{N}(0, \sigma^2 I_n)$.
Let $S$ be a fixed subspace of dimension $d$ and denote by $w_S$ and $w_{S^{\perp}}$ the projections
onto $S$ and $S^{\perp}$ respectively.
Then for an absolute constant $\alpha$, the following statements hold
\begin{itemize}
\item With probability at least $1-\exp(-\alpha t^2)$,
\[
 \sigma(\sqrt{n} - t)  \le \norm{w} \le \sigma(\sqrt{n} + t)
 \]
 \item With probability at least $(1-\exp(-\alpha t^2))^2$,
 \[
 \sigma(\sqrt{d} - t)  \le \norm{w_S} \le \sigma(\sqrt{d} + t),\qquad
 \sigma(\sqrt{n-d} - t)  \le \norm{w_{S^{\perp}}} \le \sigma(\sqrt{n-d} + t),
 \]
\end{itemize}
 \begin{proof}
   The first statement is proved in ~\citep[page 44, Equation
   3.3]{vershynin2018high}.  For the second, let $B \in \mathbb{R}^{n \times d}$
   denote an orthonormal basis for $S$ and define $y = B^T w$ Then
   $\|y\|= \|w_S\|$.  Further, $y\sim \mathcal{N}(0, \sigma^2 I_{d\times d})$
   given that $\cov(y) = \sigma^2 B^T B = \sigma^2 I_{d \times d}$. Hence, the
   bounds on $\|w_S\|$ hold with probability at least
   $p : = 1-\exp(-\alpha t^2)$ given the first statement.  By similar argument,
   the bounds on $\|w_{S^{\perp}}\|$ also hold with probability $p$.  Since
   $w_{S}$ and $w_{S^{\perp}}$ are independent, we deduce that both sets of
   bounds simultaneously hold with probability at least $p^2$.
 \end{proof}
\end{prop}

To prove \Cref{prop:OracleDenoisingUnderReach},
we decompose random noise $\sigma \epsilon$ as
\begin{align}\label{eq:noiseDecomp}
\sigma\epsilon = w_N + w_T
\end{align}
for $w_T \in \tan_{\coneName}(x_0)$ and $w_N \in \tan_{\coneName}(x_0)^{\perp}$ and use
\Cref{lem:tubularProj}. The proof follows.

  \begin{proof}[Proof of \Cref{prop:OracleDenoisingUnderReach}]
  Let $p:=1-\exp(-\alpha t^2)$.
  \Cref{prop:concenIneq} asserts that,
with probability at least $p^2$,
 \begin{align}\label{eq:noiseNormBounds}
 \sigma(\sqrt{d} - t)  \le \norm{w_T} \le \sigma(\sqrt{d} + t),\qquad
 \sigma(\sqrt{n-d} - t)  \le \norm{w_N} \le \sigma(\sqrt{n-d} + t),
 \end{align}
 These inequalities imply the claimed bounds on $\distK(x_{\sigma})$, given that
 \[
\|w_N\|  - \|w_T\|\le \distK(x_{\sigma}) \le \|w_N\|  + \|w_T\|
 \]
 by \Cref{lem:distBounds} and the fact $\dist(x_0+w_N) = \|w_N\|$
 under the reach assumption and \Cref{lem:tubularProj}.

    Using  $\proj(x_0+w_N) = x_0$, we observe that
    \begin{align*}
      \|\proj(x_\sigma)-x_0\| &= \|\proj(x_0 + w_N + w_T) - x_0\| \\
                     &= \|\proj(x_0+w_N) - x_0 + \proj(x_0+w_N + w_T) - \proj(x_0+w_N)\| \\
                     &= \|\proj(x_0+w_N) - \proj(x_0+w_N+w_T)\| \\
                     &\le C \|w_T\| \\
                     &\le C \sigma(\sqrt{d} + t)
    \end{align*}
    where the second-to-last inequality comes from \Cref{prop:Lipschitz}
    using the fact that $\reach(\coneName) > w$ and the inequalities
    $\distK(x_0+w_N) = \|w_N\| \le \|w\|$ and
    $\distK(x_0+w_N+w_T) \le \|w\|$.
  The proof is completed by dividing by our lower bound of $\distK(x_{\sigma})$.

\end{proof}

\subsection{Proof of \Cref{prop:smoothed-ideal-denoiser}} \label{sec:ideal-denoiser-derivation}
First we derive an explicit expression for the ideal denoiser for a uniform
distribution over a finite set.
\begin{lem} \label{lem:ideal-denoiser}
  When $D$ is a discrete uniform distribution over a set $\Kset$, the ideal
  denoiser $\eps^*$ is given by
  \begin{align*}
   \eps^*(x_\sigma, \sigma) = \frac{\sum_{x_0 \in \Kset} (x_\sigma-x_0) \exp(-\norm{x_\sigma-x_0}^2/2\sigma^2)}
    {\sigma \sum_{x_0 \in \Kset} \exp(-\norm{x_\sigma-x_0}^2/2\sigma^2)}.
  \end{align*}
\end{lem}

\begin{proof}
  Writing the loss explicitly as
  \begin{align*}
    \Loss_\sigma(\eps^*) =
    \int \sum_{x_0\in \Kset} \frac{1}{\abs{\Kset} \sigma \sqrt{2\pi}}
    \exp\paren{-\frac{\norm{x_\sigma-x_0}^2}{2\sigma^2}}
    \norm{(x_\sigma - x_0)/\sigma - \eps^*(x_\sigma, \sigma)}^2 d(x_\sigma),
  \end{align*}
  It suffices to take the point-wise minima of the expression inside the
  integral, which is convex in terms of $\eps^*$.
\end{proof}
From this expression of the ideal denoiser $\eps^*$, we see that $\hat x_0$ can
be written as a convex combination of points in $\Kset$:
\begin{align*}
   \hat x_0(x, \sigma) = \sigma \eps^*(x, \sigma) - x = \sum_{x_0 \in \Kset} w(x, x_0) x_0,
\end{align*}
where $\sum_{x_0 \in \Kset} w(x, x_0) = 1$.

By taking the gradient of the log-sum-exp function in the definition of
$\distop^2_\Kset(x, \sigma)$ then applying \Cref{lem:ideal-denoiser}, it is clear that
\begin{align*}
  \nabla_x \ihalf \distop^2_\Kset(x, \sigma) = \sigma \eps^*(x, \sigma).
\end{align*}
\subsection{Proof of \Cref{prop:ideal-rel-error}}
We wish to bound the error between the gradient of the true distance function
and the result of the ideal denoiser. Precisely, we want to upper-bound the
following in terms of $\distop_\Kset(x) = \norm{x - x_0^*}$, where
$x_0^* = \projK(x)$. We first define
\begin{align*}
   w(x, x') := \frac{\exp(-\norm{x-x'}^2/2\sigma^2)}{\sum_{x_0 \in \Kset} \exp(-\norm{x-x_0}^2/2\sigma^2)}.
\end{align*}
Note that by definition, we have $\sum_{x_0 \in \Kset} w(x, x_0) = 1$.
Letting $\bar N_\alpha$ denote the complement of $N_{\alpha}$ in $\Kset$, we have
\begin{align*}
  \norm{\nabla \ihalf\distop_\Kset(x)^2 - \sigma \epsilon^*(x, \sigma)}
  &= \norm{\nabla \ihalf \distop_\Kset(x)^2 - \nabla \ihalf \distop^2_\Kset(x, \sigma)} \\
  &= \textstyle{\norm{x_0^* - \sum_{x_0 \in \Kset} w(x, x_0) x_0}}\\
  &= \textstyle{\norm{\sum_{x_0 \in \Kset} w(x, x_0)(x^*_0- x_0)}}\\
     &\le \| \sum_{x_0 \in \bar N_{\alpha} } w(x, x_0)(x^*_0-x_0 ) \| + \| \sum_{x_0 \in N_{\alpha} } w(x, x_0)(x^*_0-x_0) \|\\
     &\le \| \sum_{x_0 \in \bar N_{\alpha} } w(x, x_0)(x^*_0-x_0 ) \| + C_{x, \alpha}
\end{align*}
The claim then follows from the following theorem.
\begin{thm}~\label{thm:ideal-rel-error}
Suppose $\coneName$ is a finite-set and let $x_0^* = \projK(x)$.
  Suppose we have
  \begin{align} \label{eq:alpha-cond-dist}
    \alpha \ge 1 + \frac{2\sigma^2}{\distK(x)^2} \paren{\frac{1}{e} + \log\pfrac{m}{\eta}},
  \end{align}
  then $\textstyle  \norm{\sum_{x_0 \in \bar N_\alpha} w(x, x_0) (x_0^* - x_0)} \le \eta \distK(x)$.
\end{thm}

\begin{proof}
  Applying the triangle inequality, it suffices to upper-bound each of
  $w(x, x_0) \norm{x_0^* - x_0}$. For convenience of notation let
  $\delta(x_0) := \norm{x - x_0}/\norm{x - x_0^*}$. Note that by construction
  $\delta(x_0) \ge 1$ for all $x_0 \in \Kset$, and $\delta(x_0) \ge \alpha$
  for all $x_0 \in \bar N_\alpha$. Then
  \begin{align*}
    \norm{x_0^* - x_0} &\le \norm{x_0^* - x} + \norm{x - x_0} = (1 + \delta(x_0))\norm{x - x_0^*}, \\
    w(x, x_0) &\le \exp\paren{-\frac{\norm{x-x_0}^2 - \norm{x-x_0^*}^2}{2\sigma^2}}
                \le \exp\paren{-\frac{(\delta(x_0)^2 - 1)\norm{x-x_0^*}^2}{2\sigma^2}}.
  \end{align*}
  From \eqref{eq:alpha-cond-dist} and the fact that $1/e \ge \log(a)/a$ for
  $a = \delta(x_0) + 1 \ge 1$, we have
  \begin{align*}
    \delta(x_0) - 1
    &\ge \alpha - 1
      \ge \frac{2\sigma^2}{a \norm{x-x_0^*}^2} \paren{\log(a) + \log\pfrac{m }{\epsilon}}
      = \frac{2\sigma^2}{ (\delta(x_0) + 1)\norm{x-x_0^*}^2} \log\pfrac{m (\delta(x_0) + 1)}{\epsilon} \\
    \delta(x_0)^2 - 1 &\ge \frac{2\sigma^2}{\norm{x-x_0^*}^2} \log\pfrac{m (\delta(x_0) + 1)}{\epsilon}
  \end{align*}
  Putting these together, we have:
  \begin{align*}
    { \textstyle  \norm{\sum_{x_0 \in \bar N_\alpha} w(x, x_0) (x_0^* - x_0)}}
    &\le \sum_{x_0 \in \bar N_\alpha}
      (1 + \delta(x_0))\norm{x - x_0^*}
      \exp\paren{-\frac{(\delta(x_0)^2 - 1)\norm{x-x_0^*}^2}{2\sigma^2}} \\
    &\le \sum_{x_0 \in \bar N_\alpha} \frac{\epsilon \norm{x - x_0^*}}{m} \\
    &\le \eps \distK(x)
  \end{align*}
\end{proof}

\section{DDIM with Projection Error Analysis}~\label{sec:proj-proofs}

\subsection{Proof of \Cref{thm:mainDDIM}}
We use the following lemma for gradient descent applied to the squared-distance
function $f(x)$.

\begin{lem}\label{lem:DampedGradientStep}
  Fix $x \in \mathbb{R}^n$ and suppose that $\nabla f(x)$ exists. For
  step-size $0 < \beta \le 1$ consider the gradient descent iteration applied to
  $f(x)$:
  \[
    x_{+} := x - \beta \nabla f(x)
  \]
  Then,  $\distK(x_+) = \paren{1-\beta} \distK(x) < \distK(x) $.
\end{lem}

Make the inductive hypothesis that $\dist(x_t) = \sqrt n \sigma_t$.
From the definition of DDIM~\eqref{eq:DDIM-scaled}, we have
\[
  x_{t-1} = x_t + (\frac{\sigma_{t-1}}{\sigma_t} - 1) \sigma_t \epsilon_\theta(x_t, \sigma_t).
\]
Under~\Cref{ass:denoiseIsProjection} and the inductive
hypothesis, we conclude
\begin{align*}
  x_{t-1} &= x_t + (\frac{\sigma_{t-1}}{\sigma_t} - 1) \nabla f(x_t) \\
          &= x_t - \beta_t \nabla f(x_t)
\end{align*}
Using \Cref{lem:DampedGradientStep} we have that
\begin{align*}
  \dist(x_{t-1}) = (1-\beta_t) \dist(x_t) =  \frac{\sigma_{t-1}}{\sigma_t} \dist(x_t) = \sqrt n \sigma_{t-1}
\end{align*}
The base case holds by assumption, proving the claim.

\subsection{Proof of \Cref{lem:DampedGradientStep}}
Letting $x_0 = \projK(x)$ and noting $\nabla f(x) = x - x_0$, we have
\begin{align*}
  \distK(x_+) &= \distK(x + \beta (x_0 - x))\\
              & = \norm{x + \beta (x_0 - x) - x_0} \\
              & = \norm{(x - x_0) (1-\beta)}\\
              & = (1-\beta) \distK(x)
\end{align*}

\subsection{Distance function bounds}
The distance function admits the following upper and lower bounds.
\begin{lem}~\label{lem:distBounds}
The distance function $\distK : \mathbb{R}^n \rightarrow \mathbb{R}$
for $\coneName \subseteq \mathbb{R}^n$ satisfies
\[
 \distK(u) - \|u-v\|  \le \distK(v) \le \distK(u) + \|u-v\|
\]
for all $u, v \in \mathbb{R}^n$.
\end{lem}
\begin{proof}
  By~\citep[Chapter 6, Theorem 2.1]{delfour2011shapes},
  $|\distK(u) - \distK(v)| \le \|u-v\|$, which is
  equivalent to
  \[
    \distK(u) - \distK(v) \le \|u-v\|,  \distK(v) - \distK(u) \le \|u-v\|.
  \]
  Rearranging proves the claim.
\end{proof}

\subsection{Proof of \Cref{lem:gradErrorBounds}}

We first restate the full version of \Cref{lem:gradErrorBounds}.
\begin{lem}~\label{lem:gradErrorBoundsRestate}
For $\coneName \subseteq \mathbb{R}^n$,
let $f(x) := \frac{1}{2} \distK(x)^2$.
The following statements hold.
\begin{enumerate}[label= (\alph*)]
\item\label{itm:singleGrad} If $x_{+} = x - \beta (\nabla f(x) + e)$
for $e$ satisfying $\|e\| \le \eta \distK(x)$ and $0 \le \beta \le 1$, then
\[
  (1- \beta(\eta + 1))  \distK(x)  \le \distK(x_+) \le (1+ \beta  (\eta-1))\distK(x).
\]
\item\label{itm:multstepGrad} If $x_{t-1} = x_t - \beta_t (\nabla f(x_t) + e_t)$
for $e_t$ satisfying $\|e_t\| \le \eta \distK(x_t)$ and $0 \le \beta_t \le 1$, then
\[
 \distK(x_N)\prod^{N}_{i=t} (1- \beta_i(\eta+1))  \le  \distK(x_{t-1}) \le \distK(x_N) \prod^{N}_{i=t}  (1+ \beta_i(\eta-1). )
\]
\end{enumerate}
\end{lem}
For \Cref{itm:singleGrad} we apply \Cref{lem:distBounds} at points $u = x_+$ and $v = x - \beta \nabla f(x)$.
We also use $\dist(v) = (1- \beta)  \distK(x)$, since $0 \le \beta \le 1$, to conclude that
\[
  (1- \beta)  \distK(x) - \beta \|e\| \le \distK(x_+) \le (1-\beta)\distK(x)  + \beta \|e\|.
\]
Using the assumption that $\|e\| \le \eta \distK(x)$ gives
\[
  (1- \beta - \eta \beta)  \distK(x)  \le \distK(x_+) \le (1-\beta + \eta\beta)\distK(x)
\]
Simplifying completes the proof.
\Cref{itm:multstepGrad} follows from \Cref{itm:singleGrad} and induction.

\subsection{Proof of \Cref{thm:DDIMWithError}}
We first state and prove an auxiliary theorem:
\begin{thm}\label{thm:gradWithError}
  Suppose \Cref{ass:main} holds for $\nu \ge 1$ and $\eta > 0$.  Given
  $x_N$ and $\{\beta_t, \sigma_t\}^N_{i=1}$, recursively define
  $x_{t-1} = x_t + \beta_t \sigma_t \epsilon_{\theta}(x_t, t)$ and suppose that
  $\projK(x_t)$ is a singleton for all $t$.  Finally, suppose that
  $\{\beta_t, \sigma_t\}^N_{i=1}$ satisfies
  $\frac{1}{\nu} \distK(x_N)\le \sqrt{n} \sigma_N \le \nu \distK(x_N)$ and
  \begin{align}\label{eq:parameterAssumption}
    \frac{1}{\nu} \distK(x_N) \prod^N_{i=t} (1+ \beta_i(\eta-1) )\le \sqrt{n} \sigma_{t-1}  \le \nu
    \distK(x_N)\prod^N_{i=t} (1- \beta_i(\eta+1)).
  \end{align}
  The following statements hold.
  \begin{itemize}
   \item  $\distK(x_N)\prod^N_{i=t} (1- \beta_i(\eta+1)) \le \distK(x_{t-1}) \le \distK(x_N) \prod^N_{i=t} (1+\beta_i (\eta-1) )$
   \item $\frac{1}{\nu} \distK(x_{t-1})  \le \sqrt{n} \sigma_{t-1} \le \nu \distK(x_{t-1})$
  \end{itemize}
  \begin{proof}
    Since $\projK(x_t)$ is a singleton, $\nabla f(x_t)$ exists.  Hence, the
    result will follow from \Cref{itm:multstepGrad} of
    \Cref{lem:gradErrorBoundsRestate} if we can show that
    $\|\beta_t \sigma_t \epsilon_{\theta}(x_t, t) - \nabla f(x_t)\| \le \eta
    \distK(x_t)$.  Under \Cref{ass:main}, it suffices to show that
    \begin{align}\label{eq:proofGradError}
      \frac{1}{\nu} \distK(x_t)  \le \sqrt{n}\sigma_t \le \nu \distK(x_t)
    \end{align}
    holds for all $t$.
    We use induction, noting that the base case $(t = N)$ holds by  assumption.
    Suppose then that~\eqref{eq:proofGradError} holds for all $t, t+1, \ldots, N$.
    By \Cref{lem:gradErrorBounds} and \Cref{ass:main},
    we have
    \[
      \distK(x_N)\prod^N_{i=t} (1- \beta_i(\eta+1)) \le \distK(x_{t-1}) \le \distK(x_N) \prod^N_{i=t} (1+(\eta-1)\beta_i )
    \]
    Combined with~\eqref{eq:parameterAssumption} shows
    \begin{align*}
      \frac{1}{\nu} \distK(x_{t-1})  \le \sqrt{n} \sigma_{t-1} \le \nu \distK(x_{t-1}),
    \end{align*}
    proving the claim.
  \end{proof}
\end{thm}
\Cref{thm:DDIMWithError} follows
by observing the admissibility assumption and the DDIM step-size rule,
which satisfies $\sigma_{t-1} = (1-\beta_t) \sigma_t$, implies~\eqref{eq:parameterAssumption}.

\subsection{Proof of \Cref{thm:DDIMStepSize}}
Assuming constant step-size $\beta_i = \beta$ and
dividing~\eqref{eq:DDIMParamCond} by $\prod^N_{i=1} (1-\beta)$  gives the
conditions
\[
 \paren{1 + \eta\frac{\beta}{1-\beta}}^N  \le \nu,
 \qquad \paren{1 - \eta\frac{\beta}{1-\beta}}^N  \ge \frac{1}{\nu}.
\]
Rearranging and defining $a =\eta\frac{\beta}{1-\beta}$ and $b = \nu^{\frac{1}{N}}$
gives
\[
   a \le b -1, \qquad a \le 1-b^{-1}.
\]
Since $b-1 - (1-b^{-1})  = b+b^{-1}-2 \ge 0$ for all $b >0$,
we conclude $a \le b -1$ holds if $a \le 1-b^{-1}$ holds.
We therefore consider the second
inequality $\eta\frac{\beta}{1-\beta} \le 1- \nu^{-1/N}$,
noting that it holds for all $0 \le \beta < 1$
if and only if $0\le \beta \le \frac{k}{1+k}$
for $k = \frac{1}{\eta}(1- \nu^{-1/N})$, proving the claim.

\subsection{Proof of \Cref{thm:DDIMStepSizeLimit}}
The value of $\sigma_0/\sigma_N$ follows from the definition of $\sigma_t$ and
and the upper bound for $\distK(x_0)/\distK(x_N)$ follows from
\Cref{thm:DDIMStepSize}. We introduce the parameter $\mu$ to get a general form
of the expression inside the limit:
\begin{align*}
  (1-\mu \beta_{*,N})^N = \paren{1-\mu \frac{1-\nu^{-1/N}}{\eta + 1 - \nu^{-1/N}}}^N.
\end{align*}
Next we take the limit using L'H\^opital's rule:
\begin{align*}
  \lim_{N \rightarrow \infty} \paren{1-\mu \frac{1-\nu^{-1/N}}{\eta + 1 - \nu^{-1/N}}}^N
  &= \exp\paren{\lim_{N \rightarrow \infty}
    \log \paren{1-\mu \frac{1-\nu^{-1/N}}{\eta + 1 - \nu^{-1/N}}}/(1/N)} \\
  &= \exp\paren{\lim_{N \rightarrow \infty}
    \frac{\eta\mu\log(\nu)}{(\nu^{-1/N} - \eta - 1)(\nu^{1/N} (\eta-\mu+1) + \mu-1)}} \\
  &= \exp\paren{-\frac{\mu\log(\nu)}{\eta}} \\
  &= \paren{1/\nu}^{\mu/\eta}.
\end{align*}
For the first limit, we set $\mu = 1$ to get
\begin{align*}
  \lim_{N\rightarrow \infty}(1-\beta_{*,N})^N = (1/\nu)^{1/\eta}.
\end{align*}
For the second limit, we set $\mu = 1-\eta$ to get
\begin{align*}
  \lim_{N\rightarrow \infty}(1+ (\eta-1)\beta_{*,N})^N = (1/\nu)^{\frac{1-\eta}{\eta}}.
\end{align*}

\subsection{Denoiser Error}
\Cref{ass:main} places a condition directly on the approximation of
$\nabla f(x)$, where $f(x):= \frac{1}{2} \distK(x)$, that is jointly obtained
from $\sigma_t$ and the denoiser $\epsilon_{\theta}$.  We prove this assumption
holds under a direct assumption on $\nabla \distK(x)$, which is easier to verify
in practice.
\begin{ass}\label{ass:denoiserError}
There exists $\nu \ge 1$ and $\eta > 0$ such that
if $ \frac{1}{\nu}\distK(x)  \le \sqrt{n}\sigma_t \le \nu \distK(x)$ then
$\|\epsilon_{\theta}( x, t) -  \sqrt{n}\nabla\distK(x)\| \le \eta$
\end{ass}
\begin{lem}\label{lem:denoiserErrorImpliesGrad}
If \Cref{ass:denoiserError} holds with $(\nu, \eta)$, then \Cref{ass:main}
holds with $(\hat \nu, \hat \eta)$, where $\hat \eta = \frac{1}{\sqrt n}\eta \nu  + \max(\nu-1, 1-\frac{1}{\nu})$
and $\hat \nu = \nu$.
\begin{proof}
Multiplying the error-bound on $\epsilon_{\theta}$ by $\sigma_t$
and using $\sqrt{n}\sigma_t \le \nu \distK(x)$ gives
\[
\|\sigma_t \epsilon_{\theta}( x, t) -  \sqrt{n}\sigma_t\nabla\distK(x)\| \le \eta \sigma_t \le \eta \nu \frac{1}{\sqrt n} \distK(x)
\]
Defining $C = \sqrt{n}\sigma_t - \distK(x)$ and simplifying gives
\begin{align*}
\eta \nu \frac{1}{\sqrt n} \distK(x) &\ge
\|\sigma_t \epsilon_{\theta}( x, t) -  \sqrt{n}\sigma_t\nabla\distK(x)\|\\
& =\|\sigma_t \epsilon_{\theta}( x, t) -  \nabla f(x) - C\nabla \distK(x)  \| \\
&\ge\|\sigma_t \epsilon_{\theta}( x, t) -  \nabla f(x)  \|
-  \|C\nabla \distK(x)\| \\
&=\|\sigma_t \epsilon_{\theta}( x, t) -  \nabla f(x)  \|
- |C| \\
\end{align*}
Since  $(\frac{1}{\nu} - 1) \distK(x)  \le C \le (\nu-1)\distK(x)$
and $\nu \ge 1$, the \Cref{ass:main} error
bound holds for the claimed $\hat \eta$.
\end{proof}
\end{lem}

\section{Derivation of Gradient Estimation Sampler}\label{sec:second-order-weighting}
To choose $W$, we make two assumptions on the denoising error: the coordinates
$e_t(\epsilon)_i$ and $e_t(\epsilon)_j$ are uncorrelated for all $i \ne j$, and
$e_t(\epsilon)_i$ is only correlated with $e_{t+1}(\epsilon)_i$ for all $i$.  In
other words, we consider $W$ of the form
\begin{align}\label{eq:Wform}
  W = \bmat{aI & bI \\ bI & cI}
\end{align}
and next show that this choice leads to a simple rule for selecting $\bar\epsilon$.
From the optimality conditions of the quadratic optimization problem~\eqref{eq:estimatorGeneralForm}, we get that
\begin{align*}
  \bar{\epsilon}_t = \frac{a+b}{a+c+2b} \epsilon_{\theta}(x_t, \sigma_t)
  + \frac{c+b}{a+c+2b} \epsilon_{\theta}(x_{t+1}, \sigma_{t+1}).
\end{align*}
Setting $\gamma = \frac{a+b}{a+c+2b}$, we get the update rule \eqref{eq:eps-bar-gamma}.
When $b \ge 0$, the minimizer $\bar{\epsilon}_t$ is a simple convex combination of
denoiser outputs. When $b < 0$, we can have $\gamma < 0$ or $\gamma > 1$, i.e.,
the weights in \eqref{eq:eps-bar-gamma} can be negative (but still sum to
1). Negativity of the weights can be interpreted as cancelling positively
correlated error $(b < 0)$ in the denoiser outputs.
Also note we can implicitly search over $W$ by directly searching for $\gamma$.

\section{Further Experiments}\label{sec:futher-experiments}

\subsection{Denoising Approximates Projection}
We test our interpretation that denoising approximates projection on pretrained
diffusion models on the CIFAR-10 dataset. In these experiments, we take a
50-step DDIM sampling trajectory, extract $\epsilon(x_t, \sigma_t)$ for each $t$
and compute the cosine similarity for every pair of $t, t’ \in [1,50]$. The
results are plotted in \Cref{fig:denoising-approx-proj}. They show that the
direction of $\epsilon(x_t, \sigma_t)$ over the entire sampling trajectory is
close to the first step’s output $\epsilon(x_N, \sigma_N)$.  On average over
1000 trajectories, the minimum similarity (typically between the first step when
$t=50$ and last step when $t’=1$) is ~0.85, and for the vast majority (over
80\%) of pairs the similarity is $> 0.99$, showing that the denoiser outputs
approximately align in the same direction, validating our intuitive picture in
\Cref{fig:cartoons}.

\begin{figure}
  \centering
  \includegraphics[width=0.5\textwidth]{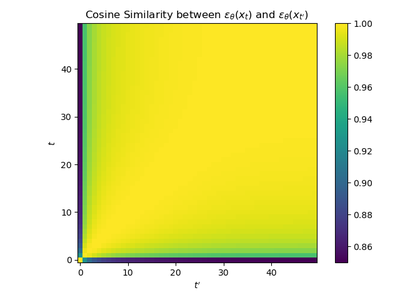}
  \caption{Plot of the cosine similarity between $\epsilon_\theta(x_t, t)$ and
    $\epsilon_\theta(x_{t'}, t')$ over $N=50$ steps of DDIM denoising on the
    CIFAR-10 dataset. Each cell is the average result of 1000 runs. }
  \label{fig:denoising-approx-proj}
\end{figure}

\begin{figure}
  \centering
  \begin{subfigure}[b]{0.45\textwidth}
    \caption{Plot of $\norm{\epsilon_\theta(x_t, \sigma_t)}/\sqrt{n}$ against $t$.}
    \includegraphics[width=\textwidth]{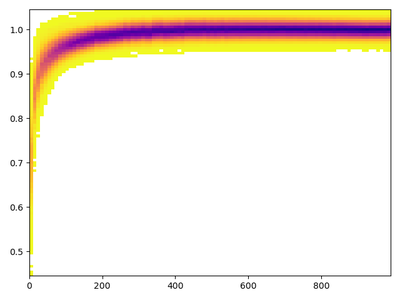}
    \label{fig:denoising-norm}
  \end{subfigure}
  \begin{subfigure}[b]{0.45\textwidth}
    \centering
    \caption{Plot of
      $\norm{\epsilon_\theta(x_0 + \sigma_t \epsilon, \sigma_t) -
        \epsilon}/\sqrt{n}$ against $t$.}
    \includegraphics[width=\textwidth]{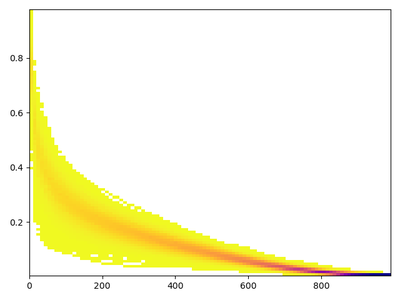}
    \label{fig:denoiser-error}
  \end{subfigure}
  \caption{Plots of the norm of the denoiser at different stages of denoising,
    as well as the ability of the denoiser to accurately predict the added noise
    as a function of noise added.}
  \label{fig:distance-fn-plots}
\end{figure}

\begin{table}[H]
  \centering
  \begin{tabular}[t]{lll}
    \toprule
                     & CIFAR-10  & CelebA \\
    \midrule
    DDIM (w/o both)  & 16.86 & 18.08 \\
    Ours (w/o sampler)  & 13.25 & 13.55 \\
    Ours (w/o schedule) & 8.30 & 11.87 \\
    Ours (with both) & \textbf{3.85} & \textbf{4.30} \\
    \bottomrule
  \end{tabular}
  \vspace{0.5cm}
  \caption{Ablation study of the effects of the schedule improvements in
    \Cref{sec:admissible-params} and the sampler improvements in
    \Cref{sec:newAlgorithms}, for $N=10$ steps.}
  \label{table:ablation}
\end{table}

\begin{table}[H]
  \centering
  \begin{tabular}[t]{lllllllll}
    \toprule
    & \multicolumn{4}{c}{CIFAR-10 FID} & \multicolumn{4}{c}{CelebA FID} \\
    DDIM Sampler      & $N=5$ & $N=10$ & $N=20$ & $N=50$ & $N=5$ & $N=10$ & $N=20$ & $N=50$ \\
    \midrule
    DDIM              & 47.21 & 16.86  & 8.28   & 4.81   & 32.21 & 18.08  & 11.81  & 7.39 \\
    DDIM Offset       & 36.09 & 14.19  & 7.51   & 4.69   & 27.79 & 15.38  & 10.05  & 6.80 \\
    EDM               & 61.63 & 20.85  & 9.25   & 5.39   & 33.00 & 16.72  & 9.78   & 6.53 \\
    Ours (Log-linear) & 40.33 & 13.37  & 6.88   & 4.71   & 28.07 & 13.63  & 8.80   & 6.79 \\
    \toprule
    & \multicolumn{4}{c}{CIFAR-10 FID} & \multicolumn{4}{c}{CelebA FID} \\
    Our Sampler       & $N=5$ & $N=10$ & $N=20$ & $N=50$ & $N=5$ & $N=10$ & $N=20$ & $N=50$ \\
    \midrule
    DDIM              & 51.65 & 8.30   & 4.96   & 3.33   & 28.64 & 11.87  & 8.00   & 4.33 \\
    DDIM Offset       & 14.95 & 7.50   & 4.58   & 3.29   & 12.19 & 9.49   & 6.58   & 4.80 \\
    EDM               & 34.26 & 4.87   & 3.64   & 3.67   & 18.68 & 5.30   & 3.95   & 4.11 \\
    Ours (Log-linear) & 12.57 & 3.79   & 3.32   & 3.41   & 10.76 & 4.79   & 4.57   & 5.01 \\
    \bottomrule
  \end{tabular}
  \vspace{0.5cm}
  \caption{Ablation of $\sigma_t$ schedules for both the DDIM and GE sampler.}
  \label{table:ablation-full}
\end{table}
We perform an ablation study on different sampling schedules. We use the four
different schedules as shown in \Cref{table:schedule-fid}:
\begin{itemize}
\item \textbf{DDIM} Default DDIM schedule with $\sigma_N = 157, \sigma_0 = 0.002$
\item \textbf{DDIM Offset} Truncated DDIM schedule starting with a smaller $\sigma$,
  with $\sigma_N = 40, \sigma_0 = 0.002$.
\item \textbf{EDM} Schedule used in \cite{karras2022elucidating} with
  $\sigma_N = 80, \sigma_0 = 0.002$.
\item \textbf{Linear} Log-linear schedule with $\sigma_N = 40$, $\sigma_1, \sigma_0$
  selected based on \Cref{sec:sigma-select}.
\end{itemize}
Our results are reported in \Cref{table:ablation-full}. Our gradient-estimation
sampler consistently outperforms the DDIM sampler for all schedules and $N$. The
\emph{DDIM Offset} schedule that starts at $\sigma_N = 40$ offers an improvement
over the \emph{DDIM} schedule for $N = 5,10,20$, but performs worse for
$N = 50$. This suggests starting from a higher $\sigma_N$ for larger $N$, which
we have done in our final evaluations.

\subsection{Distance Function Properties}
We test \Cref{ass:denoiseIsProjection} and \Cref{ass:main} on pretrained
networks. If \Cref{ass:denoiseIsProjection} is true, then
$\norm{\epsilon_{\theta}(x_t, \sigma_t)}\sqrt{n} = \norm{\nabla \distK(x_t)} =
1$ for every $x_t$ along the DDIM trajectory. In \Cref{fig:denoising-norm}, we
plot the distribution of norm of the denoiser $\epsilon_\theta(x_t, \sigma_t)$
over the course of many runs of the DDIM sampler on the CIFAR-10 model for
$N=100$ steps ($t=1000, 990, \ldots, 20, 10, 0$). This plot shows that
$\norm{\epsilon_\theta(x_t, \sigma_t)}/\sqrt{n}$ stays approximately constant
and is close to 1 until the end of the sampling process. We next test
\Cref{ass:denoiserError}, which implies \Cref{ass:main} by
\Cref{lem:denoiserErrorImpliesGrad}. We do this by first sampling a fixed noise
vector $\epsilon$, next adding different levels of noise $\sigma_t$, then using
the denoiser to predict $\epsilon_\theta(x_0 + \sigma_t \epsilon, \sigma_t)$. In
\Cref{fig:denoiser-error}, we plot the distribution of
$\norm{\epsilon_\theta(x_0 + \sigma_t \epsilon, \sigma_t) - \epsilon}/\sqrt{n}$
over different levels of $t$, as a measure of how well the denoiser predicts the
added noise.

\subsection{Choice of $\gamma$} \label{sec:gamma-choice} We motivate our choice
of $\gamma = 2$ in \Cref{alg:second-order} with the following experiment.  For
varying $\gamma$, \Cref{fig:fid-gamma} reports FID scores of our sampler on the
CIFAR-10 and CelebA models for $N=5,10,20$ timesteps using the $\sigma_t$
schedule described in \Cref{sec:sigma-select}. As shown, $\gamma \approx 2$
achieves the optimal FID score over different datasets for $N < 20$. For
sampling from the CelebA dataset, we found that setting $\gamma = 2.4$ for
$N=20$ and $\gamma=2.8$ for $N=50$ achieves the best FID results.

\begin{figure}
  \centering
  \includegraphics[width=0.5\textwidth]{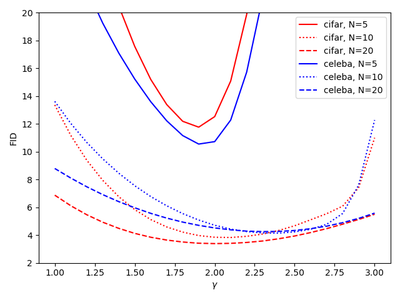}
  \caption{Plot of FID score against $\gamma$ for our second-order sampling
    algorithm on the CIFAR-10 and CelebA datasets for $N=5,10,20$ steps.}
  \label{fig:fid-gamma}
\end{figure}

\section{Experiment Details} \label{sec:exp-details}

\subsection{Pretrained Models}
The CIFAR-10 model and architecture were based on that in
\cite{ho2020denoising}, and the CelebA model and architecture were based on that
in \cite{song2020denoising}. The specific checkpoints we use are provided by
\cite{liu2022pseudo}. We also use Stable Diffusion 2.1 provided
in \url{https://huggingface.co/stabilityai/stable-diffusion-2-1}. For the
comparison experiments in \Cref{fig:sd-figures}, we implemented our gradient
estimation sampler to interface with the HuggingFace diffusers library and use
the corresponding implementations of UniPC, DPM++, PNDM and DDIM samplers with
default parameters.

\subsection{FID Score Calculation}
For the CIFAR-10 and CelebA experiments, we generate 50000 images using our
sampler and calculate the FID score using the library
in \url{https://github.com/mseitzer/pytorch-fid}. The statistics on the training
dataset were obtained from the files provided by \cite{liu2022pseudo}. For the
MS-COCO experiments, we generated images from 30k text captions drawn from the
validation set, and computed FID with respect to the 30k corresponding images.

\subsection{Our Selection of $\sigma_t$}
\label{sec:sigma-select}
Let $\sigma_1^{\text{DDIM}(N)}$ be the noise level at $t=1$ for the DDIM sampler
with $N$ steps.  For the CIFAR-10 and CelebA models, we choose
$\sigma_1 = \sqrt{\sigma_1^{\text{DDIM}(N)}}$ and $\sigma_0 = 0.01$. For
CIFAR-10 $N=5,10,20,50$ we choose $\sigma_N = 40$ and for CelebA $N=5,10,20,50$
we choose $\sigma_N = 40,80,100,120$ respectively. For Stable Diffusion, we use
the same sigma schedule as that in DDIM.

\subsection{Text Prompts}
For the text to image generation in \Cref{fig:sd-figures}, the text prompts used
are:
\begin{itemize}
\item ``A digital Illustration of the Babel tower, 4k, detailed, trending in artstation, fantasy vivid colors''
\item ``London luxurious interior living-room, light walls''
\item ``Cluttered house in the woods, anime, oil painting, high resolution, cottagecore, ghibli inspired, 4k''
\end{itemize}

\end{document}